\documentclass[letterpaper]{article} % DO NOT CHANGE THIS
\usepackage{aaai2026}  % DO NOT CHANGE THIS
\usepackage{times}  % DO NOT CHANGE THIS
\usepackage{helvet}  % DO NOT CHANGE THIS
\usepackage{courier}  % DO NOT CHANGE THIS
\usepackage[hyphens]{url}  % DO NOT CHANGE THIS
\usepackage{graphicx} % DO NOT CHANGE THIS
\usepackage{natbib}  % DO NOT CHANGE THIS AND DO NOT ADD ANY OPTIONS TO IT
\usepackage{caption} % DO NOT CHANGE THIS AND DO NOT ADD ANY OPTIONS TO IT
\usepackage{algorithm}
\usepackage[noend]{algpseudocode}

\usepackage{mathtools}
\usepackage{amsmath}
\usepackage{amsthm}
\usepackage{amsfonts}
\usepackage{amssymb}

\usepackage{enumitem}
\usepackage{booktabs}
\usepackage{xcolor}
\usepackage{subfig}

\newcommand{\ignore}[1]{}

\newcommand{\bmu}{{\boldsymbol{\mu}}}

\newtheorem{lemma}{Lemma}
\newtheorem{theorem}{Theorem}

\usepackage[textsize=small,textwidth=3cm]{todonotes}
\definecolor{kentuckyblue}{RGB}{0, 93, 170}

\usepackage{newfloat}
\usepackage{listings}
\DeclareCaptionStyle{ruled}{labelfont=normalfont,labelsep=colon,strut=off} % DO NOT CHANGE THIS
\floatstyle{ruled}
\newfloat{listing}{tb}{lst}{}
\floatname{listing}{Listing}
\nocopyright
\title{Fair Algorithms with Probing for Multi-Agent Multi-Armed Bandits}
\author{
    Tianyi Xu\textsuperscript{\rm 1},
    Jiaxin Liu\textsuperscript{\rm 2},
    Nicholas Mattei\textsuperscript{\rm 1},
    Zizhan Zheng\textsuperscript{\rm 1},
}
\affiliations{
    \textsuperscript{\rm 1}Department of Computer Science, Tulane University, New Orleans, LA, USA\\
    \textsuperscript{\rm 2}Siebel School of Computing and Data Science, University of Illinois Urbana-Champaign, Urbana, IL, USA\\
    \{txu9, nsmattei, zzheng3\}@tulane.edu, jiaxin26@illinois.edu
}

\usepackage{bibentry}
\begin{document}

\maketitle

\begin{abstract}
We propose a multi-agent multi-armed bandit (MA-MAB) framework to ensure fair outcomes across agents while maximizing overall system performance. For example, in a ridesharing setting where a central dispatcher assigns drivers to distinct geographic regions, utilitarian welfare (the sum of driver earnings) can be highly skewed—some drivers may receive no rides. We instead measure fairness by Nash social welfare, i.e., the product of individual rewards. A key challenge in this setting is decision-making under limited information about arm rewards (geographic regions). To address this, we introduce a novel probing mechanism that strategically gathers information about selected arms before assignment. In the offline setting, where reward distributions are known, we exploit submodularity to design a greedy probing algorithm with a constant-factor approximation guarantee. In the online setting, we develop a probing-based algorithm that achieves sublinear regret while preserving Nash social welfare. Extensive experiments on synthetic and real-world datasets demonstrate that our approach outperforms baseline methods in both fairness and efficiency.
\end{abstract}


\section{Introduction}
The multi-agent multi-armed bandit (MA-MAB) framework models a scenario where \(M\) agents compete for \(A\) arms over discrete rounds. In each round, a centralized decision-maker assigns each agent to an arm and observes the individual reward returned by each agent–arm pair, which can then be aggregated to measure total system performance. Typically this aggregation is measured by the utilitarian welfare, or sum of total rewards.
%The multi-agent multi-armed bandit (MA-MAB) framework models a scenario where $M$ agents compete for $A$ arms over a discrete set of time steps or rounds. In each round, a centralized decision-maker assigns arms to the agents and observes the corresponding rewards. 
%\nick{total or per agent?} 
One such application is ridesharing, where a central dispatcher assigns drivers (agents) to geographic regions (arms) based on estimated demand and driver availability. The dispatcher then observes the individual reward obtained by each driver in its assigned region.
%One such application is ridesharing, where drivers (agents) independently choose geographic regions (arms) to serve based on factors such as demand and availability, which directly influence their rewards. 
%\nick{who does the assignment in this setting?}

Maximizing total expected reward is a common objective in the MA-MAB framework, but it can lead to inequalities in practice \citep{NIPS2016_eb163727}. Optimizing for aggregate performance often results in a concentration of profitable arms among a few agents, disadvantaging others \citep{pandp.20181018,pmlr-v32-agarwalb14}. This is especially concerning in applications where fair resource assignment is essential, such as in ridesharing platforms, where drivers need equitable access to profitable areas \citep{bubeck2012, Lattimore_Szepesvári_2020}, and in content recommendation systems, where creator exposure should not be monopolized \cite{abdollahpouri2020multistakeholder}.

Many MA-MAB studies aim to improve fairness, often by altering assignment strategies or adding constraints to naive utilitarian objectives \citep{Meritocratic2018,Patil_Ghalme_Nair_Narahari_2020}.
A natural baseline is to maximize the sum of agent rewards (utilitarian welfare), but this can produce highly imbalanced assignments: profitable arms tend to be concentrated on a few agents while others receive little or nothing. This effect—where some agents are effectively “starved” of reward despite high aggregate system utility—is well documented in assignment and scheduling literature and is referred to as starvation. Maximizing total reward without regard to its distribution can therefore systematically disadvantage certain agents even when the system as a whole appears efficient.
To address this, a line of work replaces or augments the sum objective with fairness-aware criteria such as Nash Social Welfare (NSW), which takes the product (equivalently the geometric mean) of individual utilities, thereby discouraging assignments that leave any agent with very low reward and yielding a balanced equity-efficiency trade-off \citep{NEURIPS2018_be3159ad,Zhang_Conitzer_2021}. In particular, \citet{jones2023efficient} demonstrated that optimizing NSW in MA-MAB settings can prevent persistent exclusion and achieve simultaneous improvements in fairness and overall utility.

%Many MA-MAB studies aim to improve fairness, often by modifying assignment strategies or imposing fairness constraints on overall performance \citep{Meritocratic2018,Patil_Ghalme_Nair_Narahari_2020，Zoltners1983,Zoltners1980}. 
%\nick{Note that you should make a clear distinction here that the sum or utilitarian objective can lead to unfair results, as it does not take into account the distribution of rewards to individual agents. So some agents may always get 0 reward -- this is called starvation in the old school assignment literature} 
%Others adopt fairness-aware objectives, such as the Nash Social Welfare (NSW) criterion, which balances rewards among agents by penalizing assignments where any agent’s reward is too low \citep{NEURIPS2018_be3159ad,Zhang_Conitzer_2021}. Notably, \citet{jones2023efficient} demonstrated that applying NSW 
%\nick{as the overall objective?} 
%in MA-MAB frameworks can balance efficiency and fairness, preventing any agent from being consistently left out of rewarding arms.

However, a common limitation of existing approaches is their dependence on instantaneous reward feedback for updating estimates and steering assignment policies, even though in many real-world scenarios key information is either unobserved or noisy \citep{li2016collaborative}. For instance, in ridesharing platforms, uncertainty about passenger demand and fluctuating road conditions can corrupt the platform's estimates of driver-region rewards. Those corrupted estimates then propagate into the assignment step, potentially producing assignments that are unfair—for example, by repeatedly starving certain drivers who are misestimated as low-reward. To mitigate such risks, probing can be used to actively gather targeted information before making assignments \citep{NEURIPS2019_36d75342}. Originating in economics \citep{pandora-Weitzman} as a cost-bearing method to reduce uncertainty in sequential decision-making, probing has been adapted in domains such as database query optimization \citep{deshpande2016approximation,liu2008near}, real-time traffic-aware vehicle routing \citep{bhaskara2020adaptive,xu2025online}, and wireless network scheduling \citep{abs-2108-03297,xu2023online}.

%However, a common limitation of existing approaches is their reliance on instantaneous reward feedback  to update estimates and guide assignment decisions, whereas, in many practical settings, critical information remains unknown \citep{li2016collaborative}. For example, in ridesharing, uncertainties in passenger requests and real-time road conditions can significantly compromise reward estimates and fairness-oriented strategies 
%\nick{maybe: compromise reward estimates leading to assignments that may be unfair}
%In such uncertain environments, probing offers a promising approach to gather additional information before making decisions \citep{Blum2020,NEURIPS2019_36d75342}. Probing was originally studied in economics \citep{pandora-Weitzman} as a cost-incurring method for acquiring additional information to improve decision-making under uncertainty. It has since been applied in database query optimization \citep{deshpande2016approximation,liu2008near}, real-time traffic monitoring for vehicle routing \citep{bhaskara2020adaptive}, and wireless network scheduling \citep{abs-2108-03297,xu2023online}.

Our work extends the NSW-based multi-agent multi-armed bandit (MA-MAB) framework by incorporating a probing mechanism to gather extra information, refining reward estimates and improving the exploration-exploitation balance while ensuring fairness. The decision-maker first probes a subset of arms for detailed reward data, then assigns agents fairly according to an objective that incorporates a fairness measure. In the offline setting, where reward distributions are known, we develop a greedy probing algorithm with a provable performance bound, leveraging the submodular structure of our objective. For the online setting, we propose a combinatorial bandit algorithm with a derived regret bound.

Integrating probing into MA-MAB poses the challenge of balancing exploration and exploitation while maintaining fairness. Related probing problems have been shown to be NP-hard \citep{goel2006asking}, and while previous works \citep{Zuo_Zhang_Joe-Wong_2020}, have explored probing strategies in MA-MAB, they have simplified assumptions, such as limiting rewards to Bernoulli distributions and ignoring fair assignment. Our framework overcomes these limitations by considering general reward distributions, ensuring fairness, and introducing a probing budget to optimize performance under exploration constraints.

The primary contributions of our work are as follows:
   
   (1) We extend the multi-agent MAB framework with a novel probing mechanism that tests selected arms before assignment. This approach ensures fairness through Nash Social Welfare optimization, departing from previous work that focuses solely on the sum of rewards \citep{Zuo_Zhang_Joe-Wong_2020}.
   
   (2) For known reward patterns in offline setting, we develop a simple yet effective greedy probing strategy with provable performance guarantees, while maintaining fairness across agents.
   
   (3) For the online setting where rewards are unknown, we propose an algorithm that balances exploration and fairness, proving that probing and fair assignment do not compromise asymptotic performance.
   
   (4) Experiments on both synthetic and real-world data demonstrate that our method achieves superior performance compared to baselines, validating the effectiveness of the probing strategy and our algorithm.

\section{Related Work}

The multi‑armed bandit (MAB) framework has been key for sequential decision‑making under uncertainty \citep{LAI19854,garivier2013,892067}. While early MAB models involve a single decision‑maker choosing one arm per round, many real‑world problems involve multiple agents acting simultaneously \citep{NEURIPS2019_85353d3b}, each potentially choosing different arms \citep{NEURIPS2021_c96ebeee}. Most existing multi‑agent MAB methods focus on maximizing the total sum of rewards, which can unfairly favor some agents. To address this, researchers have proposed fairness‑aware approaches \citep{NIPS2016_eb163727}, with Nash Social Welfare (NSW) \citep{Kaneko1979} proving effective because it maximizes the product of agents’ utilities.

In practice, key aspects of reward distributions are often unknown \citep{8895728,Lattimore_Szepesvári_2020}. Existing fair MA‑MAB methods generally rely on passive feedback \citep{Liu_2010,Gai2014DistributedSO}. Probing, or “active exploration,” seeks extra information by testing a subset of arms before committing \citep{Chen_Javdani_2015,NIPS2014_66368270}, and is valuable when exploring poorly understood arms is risky \citep{golovin_krause_2011,bhaskara2020adaptive}. 

Recent single‑agent studies formally introduce probing costs: \citet{tucker2023costlyobs} analyze bandits with costly reward observations, giving matching $\Theta(c^{1/3}T^{2/3})$ bounds; \citet{elumar2024costlyprobe} allow paying to probe one arm per round and achieve $\tilde{\mathcal O}(\sqrt{KT})$ regret; and Observe‑Before‑Play bandits \citep{Zuo_Zhang_Joe-Wong_2020,zhang2020conversational,zuo2022hierarchical} permit a limited number of pre‑observations each round. Offline work has also examined submodular probing for routing problems \citep{bhaskara2020adaptive}. However, the bulk of probing research still optimizes aggregate metrics—coverage, latency, or total reward—without incorporating inter‑agent equity. This leaves a gap between fair multi‑agent MAB methods that rely on passive feedback and probing strategies that ignore fairness. We close this gap by coupling cost‑aware, submodular probe selection with NSW‑oriented assignment, so information gained through active exploration translates directly into fair outcomes for all agents.

\section{Problem Formulation}

In this section, we define the fair multi-agent multi-armed bandit (MA-MAB) problem by extending the classical multi-armed bandit framework to incorporate fairness, multi-agent interactions, and the effect of probing decisions. The goal is to optimize both fairness and utility while accounting for probing overhead.

\subsection{Agents, Arms, and Rewards}
\label{sec:agents}

We consider a set of \(M\) \emph{agents}, indexed by \([M]=\{1,\dots,M\}\), and a set of \(A\) \emph{arms}, indexed by \([A]=\{1,\dots,A\}\).  For every pair \((j,a)\) with \(j\in[M]\) and \(a\in[A]\), let \(D_{j,a}\) denote an unknown reward distribution with cumulative distribution function \(F_{j,a}\) and mean \(\mu_{j,a}\in[0,1]\).

At each round \(t\in[T]\), the decision-maker first selects a \emph{probing set} \(S_t\subseteq[A]\), receiving for each \(j\in[M]\) and \(a\in S_t\) a fresh reward \(R_{j,a,t}\) drawn i.i.d.\ from \(D_{j,a}\).  If \(a\notin S_t\), it instead relies on the current estimate \(\mu_{j,a}\) (true mean in offline analysis).  Using these observed rewards and estimates, the decision-maker then assigns each agent \(j\) to an arm \(a_{j,t}\in[A]\). 

Let $R_t = \{R_{j,a,t}\mid j\in[M], a\in S_t\}$ denote the set of rewards revealed by probing in round \(t\). All rewards lie in \([0,1]\).

%At each round \(t\in[T]\), the decision‑maker first chooses a \emph{probing set} \(S_t\subseteq[A]\) 
%\nick{Make clear here: and the decision maker gets XYZ, using this information to assign..} 
%and then assigns each agent \(j\) an arm \(a_{j,t}\in[A]\).  
%If \(a\in S_t\), a fresh reward \(R_{j,a,t}\) drawn i.i.d. from \(D_{j,a}\) is revealed; otherwise the learner relies on the current estimate \(\mu_{j,a}\) (true mean in offline analysis).  Let
%\[
%R_t \;=\; \{\,R_{j,a,t} \mid j\in[M],\,a\in[A]\,\}
%\]
%denote the (partially observed) reward matrix at time \(t\).  All rewards lie in \([0,1]\).

\paragraph{Illustrative Mappings.}
\textbf{(a) Ridesharing.}  
Agents \(j\) are drivers; arms \(a\) are pickup zones obtained by a \(0.01^{\circ}\!\times\!0.01^{\circ}\) city grid.  
Before dispatching, the platform may probe a handful of zones (querying live app pings) to refine demand estimates, then assign drivers under a fairness objective.  \textbf{(b) 60 GHz WLAN Scheduling.}  
Here agents are client devices, arms are access‑points, and probing corresponds to brief beam‑sounding measurements before scheduling transmissions.

\subsection{Fairness Objective: Nash Social Welfare}

To balance efficiency and equity, we maximize the (expected) \emph{Nash Social Welfare} (NSW).  
An \emph{assignment policy} at round \(t\) is a matrix \(\pi_t=[\pi_{j,a,t}]\in[0,1]^{M\times A}\) where \(\pi_{j,a,t}\) is the probability that agent \(j\) receives arm \(a\).  Given \(S_t\), the realized rewards \(R_t\), and the mean matrix \(\bmu=[\mu_{j,a}]\), we define the instantaneous NSW as
\[
\mathrm{NSW}(S_t, R_t, \bmu, \pi_t) = \prod_{j\in [M]} 
\left(
\begin{split}
&\sum_{a\in S_t} \pi_{j,a,t}\,R_{j,a,t}\\[1mm]
&\quad + \sum_{a\notin S_t} \pi_{j,a,t}\,\mu_{j,a}
\end{split}
\right).
\]
so each agent’s utility contributes multiplicatively, discouraging assignments that leave any agent with a very low expected reward (e.g., a driver stranded without passengers).

The policy must satisfy
\[
\sum_{a}\pi_{j,a,t}\le1 \quad(\forall\,j),\qquad
\sum_{j}\pi_{j,a,t}\le1 \quad(\forall\,a),
\]
%ensuring, in expectation, one arm per agent and no arm overbooked.
ensuring at most one arm per agent and no arm is overbooked. As in \citet{jones2023efficient}, \(\pi_{j,a,t}\) may be fractional, reflecting randomized assignment commonly used in online platforms.

\paragraph{Why Nash Social Welfare?}
We adopt \emph{Nash Social Welfare} (NSW) as our fairness objective for three key reasons:  
(i) \textbf{Pareto efficiency} —maximizing the geometric mean never sacrifices total reward when a Pareto‑improvement is possible;  
(ii) \textbf{Scale invariance} — multiplying all utilities by the same constant leaves the maximizer unchanged, preventing bias due to units of measurement;  
(iii) \textbf{Balanced equity‑efficiency trade‑off} — NSW penalizes inequality more gently than max–min fairness while still discouraging highly skewed assignments, offering a smooth continuum between purely utilitarian and purely egalitarian objectives \citep{Caragiannis2019,THOMSON2011393}.  

%These properties make NSW particularly suitable for multi‑agent platforms—such as ridesharing dispatchers or WLAN schedulers—where both overall throughput and individual satisfaction matter.

\subsection{Probing Overhead}
Probing incurs a cost that increases with the size of the probing set. Since probing provides additional reward realizations but also consumes resources, we model the effective (instantaneous) reward at time \(t\) as follows:
\[
\mathcal{R}_t^{\mathrm{total}} = \Bigl(1-\alpha(|S_t|)\Bigr)\,
\mathbb{E}_{R_t}\Bigl[\mathrm{NSW}(S_t, R_t, \bmu, \pi_t)\Bigr],
\]
where \(\alpha: \{0,1,\ldots,I\}\to [0,1]\) is a non-decreasing overhead function satisfying \(\alpha(0)=0\) and \(\alpha(I)=1\). Here, the expectation \(\mathbb{E}_{R_t}[\cdot]\) is taken over the reward realizations \(R_t\) (with rewards drawn i.i.d. from the corresponding \(D_{j,a}\)). 

This formulation accounts for the trade-off between exploration and exploitation. When more arms are probed, the decision-maker obtains more accurate information about reward distributions, leading to better assignment decisions. However, probing incurs costs, such as time delays, energy consumption, or computational overhead, which reduce the net benefit. For instance, in a wireless scheduling scenario, probing more channels provides better estimates of channel conditions but increases latency, reducing the system’s effective throughput~\citep{abs-2108-03297}. This formulation is appropriate because the probing impact often scales with the system’s overall performance, and similar scaling effects have been modeled in prior work~\citep{abs-2108-03297, xu2025online}. The function \(\alpha(|S_t|)\) captures this diminishing return, ensuring that excessive probing is discouraged.

\subsection{Decision Problem}

At each round \(t\), the decision-maker proceeds in two stages:

\textbf{Probing Stage:} Select a probing set \(S_t \subseteq [A]\) that balances exploration (gathering new information) and exploitation (leveraging current knowledge). Upon selecting \(S_t\), the system probes the corresponding arms and obtains their reward realizations. That is, for each \(a\in S_t\) and each agent \(j\in [M]\), a reward \(R_{j,a,t}\) is sampled i.i.d. from the distribution \(D_{j,a}\).

\textbf{Assignment Stage:} Based on the observed rewards \(R_t\) for arms in \(S_t\) (and the known mean rewards \(\mu_{j,a}\) for arms not in \(S_t\)), choose a probabilistic assignment policy \(\pi_t \in \Delta^A\) to assign arms to agents.

%\red{This should be moved to where it first appears in Section 3.3.} 
The decision-maker's goal is to select \((S_t,\pi_t)\) in each round so as to maximize \(\mathcal{R}_t^{\mathrm{total}}\) and thereby achieve sublinear cumulative regret.

\subsection{Regret and Performance Measure}
To evaluate the performance of our online learning approach, we define regret by comparing the achieved reward with the optimal reward obtained in an offline setting. 

In the offline setting, the decision-maker has full knowledge of $\bmu$ and can directly compute the optimal probing set and assignment policy \((S_t^*,\pi_t^*)\) that maximizes the expected NSW objective:
\[
(S_t^*, \pi_t^*) = \arg\max_{S, \pi} \Bigl(1-\alpha(|S|)\Bigr) \mathbb{E}_{R_t}\Bigl[\mathrm{NSW}(S, R_t, \bmu, \pi)\Bigr].
\]
This serves as a performance benchmark.

In contrast, the online setting requires the decision-maker to learn \(\bmu\) over time while making sequential decisions based on observed rewards. The cumulative regret measures the performance gap between the online strategy and the offline optimal policy:
\begin{equation}\label{regret}
\begin{aligned}
\mathcal{R}_{\mathrm{regret}}(T)
&= \sum_{t=1}^T \Bigl[\,
   \Bigl(1-\alpha\bigl(|S_t^*|\bigr)\Bigr)\, \\
&\cdot \; \mathbb{E}_{R_t}\Bigl[\mathrm{NSW}(S_t^*, R_t, \bmu, \pi_t^*)\Bigr]
     \;-\;\mathcal{R}_t^{\mathrm{total}} 
\Bigr].
\end{aligned}
\end{equation}

Efficient algorithms aim to ensure that \(\mathcal{R}_{\mathrm{regret}}(T)\) grows sublinearly with \(T\), thereby balancing fairness, overall utility, and the cost of probing.

\section{The Offline Setting}

In the offline setting, all reward distributions are known in advance, reducing the problem to a static optimization over the probing set \(S\) without the time index \(t\). Given any fixed \(S\), the optimal assignment policy \(\pi\) \citep{cole2015approximating} can be computed, allowing the effective reward to be expressed as a function of \(S\).

This optimization is computationally challenging, as similar problems have been shown to be NP-hard \citep{goel2006asking}. To address this, we develop a greedy algorithm based on submodular maximization techniques to obtain an approximate solution while accounting for probing costs.

\subsection{Optimization Objective and Probing Utility}

Since the offline setting is static, we drop the time index \(t\). The decision-maker is now tasked with selecting a probing set \(S \subseteq [A]\) from the \(A\) available arms. For each agent \(j\in [M]\) and each arm \(a\in [A]\), let \(D_{j,a}\) denote the known reward distribution with CDF \(F_{j,a}\) and mean \(\mu_{j,a}\). When an arm \(a\) is probed (i.e. \(a\in S\)), a reward realization \(R_{j,a}\) is observed (drawn i.i.d. from \(D_{j,a}\)); otherwise, the decision-maker relies on the mean reward \(\mu_{j,a}\).

Given a probing set \(S\), an assignment policy \(\pi = [\pi_{j,a}]\) can be applied to assign arms to agents. For any fixed \(S\), one may compute the optimal assignment policy that maximizes the expected Nash Social Welfare \citep{cole2015approximating}. Hence, we define the effective objective as
\[
\mathcal{R}(S) = \Bigl(1-\alpha(|S|)\Bigr) \cdot \mathbb{E}_{R}\Bigl[\mathrm{NSW}(S, R, \bmu, \pi^*(S))\Bigr],
\]
\(\pi^*(S)\) denotes the optimal assignment policy given \(S\), and \(\alpha(|S|)\) is the probing overhead function.

%Directly optimizing \( \mathcal{R}(S) \) is challenging due to both the combinatorial nature of the set \( S \) and the multiplicative form of \(\mathrm{NSW}\). 

Directly optimizing \( \mathcal{R}(S) \) is challenging due to both the combinatorial nature of the set \( S \) and the multiplicative form of \(\mathrm{NSW}\). To address this, we decompose \( \mathcal{R}(S) \) into two components.

\paragraph{Defining a Simplified Utility \( g(S) \).}  
To isolate the contribution of probed arms and simplify the multiplicative structure, we define
\[
g(S) \;=\; \max_{\pi \in \Delta^A_S} \;\mathbb{E}\!\Biggl[
\prod_{j \in [M]} \Bigl(\sum_{a \in S} \pi_{j,a} \, R_{j,a}\Bigr)
\Biggr],
\]
where 
\[
\begin{aligned}
\Delta^A_S = \Bigl\{ \pi \in \mathbb{R}_+^{M \times A} \mid 
& \pi_{j,a} = 0, \quad \forall a\notin S, \; \forall j \in [M], \\
& \sum_{a \in S} \pi_{j,a} \leq 1, \quad \forall j \in [M] 
\Bigr\}.
\end{aligned}
\]

Since only arms in \(S\) yield random rewards, this formulation ensures that \(g(S)\) is naturally monotonic in \(S\). 

%\red{The last equation should be an inequality?}

\paragraph{Defining the Non-probing Utility \( h(S) \).}  
Complementarily, we define the non-probing utility as
\[
h(S) \;=\; \max_{\pi \in \Delta^A_{[A]\setminus S}} \prod_{j \in [M]} \left( \sum_{a \notin S} \pi_{j,a}\, \mu_{j,a} \right),
\]
where 
\[
\begin{aligned}
\Delta^A_{[A]\setminus S} = \Bigl\{ \pi \in \mathbb{R}_+^{M \times A} \mid 
& \pi_{j,a} = 0, \quad \forall a\in S, \; \forall j \in [M], \\
& \sum_{a \in [A] \setminus S} \pi_{j,a} \leq 1, \quad \forall j \in [M] 
\Bigr\}.
\end{aligned}
\]

This formulation captures the baseline utility achievable by assigning exclusively among the unprobed arms. Similar to \( g(S) \), the assignment policy ensures that each agent's total assignment probability does not exceed one, allowing partial assignments across multiple arms.

%\red{The last equation should be an inequality?}

%\paragraph{Revised problem formulation.}
%The original objective now combines both probed and non-probed utilities: 
%\red{Is this equivalent to the original objective or is it just an upper bound? Is it possible that two arms, one from $S$ and one from $[A] \backslash S$, are assigned to the same agent?}

%\orange{Sorry, indeed, this part should not be here, and I have deleted it. I added lemma 5 about g(S)+h(S)}
%\begin{align}
%\mathcal{R}(S) \;=\; \bigl(1 - \alpha(|S|)\bigr) \cdot \Bigl[\, g(S) + h(S) \,\Bigr].
%\end{align}

\paragraph{Log Transformation and Piecewise‑Linear Approximation.}
Taking logarithms converts the product $\prod_{j}(\dots)$ into the sum
$\log g(S)$.  Unfortunately, the resulting set function is still
\emph{non‑additive}—its marginal gain depends on the current value
$g(S)$—and generally non‑submodular, so classical greedy guarantees no
longer apply.  Even a much simpler problem is already intractable:
selecting at most $k$ items to maximise a strictly concave, increasing
function of their total weight is NP‑hard to approximate within any
constant factor \citep{AhmedAtamturk2017}.  Consequently, directly
optimising $\log g(S)$ is computationally challenging.  To regain
tractability, we adopt the classical idea of piecewise‑linear upper
envelopes for concave functions\citep{HorstTuyBook,TawarmalaniSahinidisBook}.
Specifically, we construct a piecewise-linear, concave, and non-decreasing function $\phi:[0,x_{\max}]\to\mathbb{R}$ such that for all $x>0$, we have $\phi(x) \ge \log(x)$. We then define $f_{\mathrm{upper}}(S) = \phi\bigl(g(S)\bigr)$.
This construction yields an objective that (i) upper-bounds \(\log\bigl(g(S)\bigr)\) and (ii) exhibits diminishing returns, as explained next.

\subsection{Piecewise-Linear Upper Bound Construction}

To approximate \(\log\bigl(g(S)\bigr)\) in a tractable manner, we construct \(\phi\) as follows:
\begin{itemize}[leftmargin=*]
    \item \textbf{Upper Bound Assumption:} Assume that $x_{\max}$ is an upper bound on $g(S)$, i.e., for all $S$, $g(S) \le x_{\max}$.
    \item \textbf{Partitioning:} Partition the interval $[0,x_{\max}]$ into segments with breakpoints $\tau_0, \tau_1, \dots, \tau_L$ (with $0 < \tau_0 < \tau_1 < \cdots < \tau_L = x_{\max}$). Moreover, choose the partition sufficiently fine so that there exists a constant $\eta > 0$ (with $\eta$ sufficiently small) such that $\tau_{i+1} - \tau_i \le \eta$ for all $i$, and, importantly, for any $S \subseteq [A]$ and any arm $a\notin S$, $g(S\cup\{a\}) - g(S) \le \eta$. This guarantees that for every $S$ and $a$, both $g(S)$ and $g(S\cup\{a\})$ lie within the same linear segment of $\phi$.
    \item \textbf{Tangent Lines:} For each breakpoint $\tau_i$, define the tangent line $T_{\tau_i}(z) = \log(\tau_i) + \frac{1}{\tau_i}(z-\tau_i)$. By the concavity of $\log(\cdot)$, for every $z \in [\tau_i,\tau_{i+1}]$ we have $\log(z) \le T_{\tau_i}(z)$.
\end{itemize}
We then define $\phi(z) = \max_{0\le i < L} T_{\tau_i}(z)$, so that for all $z>0$, $\phi(z) \ge \log(z)$.
By construction, \(\phi\) is concave and non-decreasing. In particular, if \(x \in [\tau_i,\tau_{i+1}]\) and \(y \in [\tau_j,\tau_{j+1}]\) with \(x < y\) (so that \(\tau_i \le \tau_j\)), the slopes on these segments are given by \(1/\tau_i\) and \(1/\tau_j\), respectively. Since \(\tau_i < \tau_j\), we have $\frac{1}{\tau_i} \ge \frac{1}{\tau_j}$. 
Thus, for any \(0<x<y\le x_{\max}\) and any increments \(d,d'\ge 0\) (with \(d,d'\le \eta\) so that \(x+d\) and \(y+d'\) lie within single linear segments), it holds that
\[
\frac{\phi(x+d)-\phi(x)}{d} = \frac{1}{\tau_i} \ge \frac{1}{\tau_j} = \frac{\phi(y+d')-\phi(y)}{d'}.
\]

\noindent \textbf{Additional Assumption.} In order to handle the case when the increments differ (i.e., when $d < d'$), we assume (see Appendix for a detailed proof) that the partition is sufficiently fine so that if $x \in [\tau_i,\tau_{i+1}]$ and $y \in [\tau_j,\tau_{j+1}]$ with $\tau_i \le \tau_j$, then $\frac{d}{d'} \ge \frac{\tau_i}{\tau_j}$. In addition, for analytical convenience, we view the offline benchmark through the allocation classes associated with probed and unprobed arms.

%Finally, define the composed function $f_{\mathrm{upper}}(S) = \phi\bigl(g(S)\bigr)$.By the properties above, we will show in Lemma \ref{submodular} that \(f_{\mathrm{upper}}(S)\) is submodular. This submodularity property enables us to analyze the greedy probing strategy and derive its approximation guarantee. Specific details are shown in the Appendix in the full version of the paper~\cite{xu2025full}. 

Finally, define the composed function $f_{\mathrm{upper}}(S)=\phi(g(S))$.
In Lemma~\ref{submodular}, we show that $f_{\mathrm{upper}}(S)$ is submodular, which enables the approximation guarantee of our greedy probing strategy. The proof is provided in the Appendix.
%in the full version~\cite{xu2025full}.

\subsection{Submodularity Properties}
%We have the following (proofs are shown in the appendix): 
Utilizing the construction above, we establish the following key properties (proofs are in the Appendix). These lemmas are fundamental for proving Theorem~\ref{thm:offline-approx}, as they help establish the submodularity and approximation guarantees of Algorithm~\ref{alg:offline-greedy}.

\begin{lemma}[Monotonicity of $g(S)$]\label{lemma:monotonicity_g}
For any $S \subseteq T \subseteq [A]$, we have $g(S) \le g(T)$.
\end{lemma}

\begin{lemma}[Monotonicity]\label{lemma:monotonicity}
For any \(S \subseteq T \subseteq [A]\), we have \( f_{\text{upper}}(S) \le f_{\text{upper}}(T) \).
\end{lemma}

\begin{lemma}[Submodularity]\label{submodular}
For any \( S \subseteq T \subseteq [A] \) and \( a \notin T \),
\[
f_{\text{upper}}(S \cup \{a\}) - f_{\text{upper}}(S)
\;\;\ge\;\;
f_{\text{upper}}(T \cup \{a\}) - f_{\text{upper}}(T).
\]
\end{lemma}

\begin{lemma}[Monotonicity of $h(S)$]\label{lemma:monotonicity_h}
For any $S \subseteq T \subseteq [A]$, we have $h(S) \ge h(T)$.
\end{lemma}

\begin{lemma}\label{lemma:upper_bound}
For any probing set \( S \), we have:
\begin{align*}
\mathcal{R}(S) &= (1 - \alpha(|S|)) \mathbb{E}_R\Bigl[\mathrm{NSW}(S, R, \bmu, \pi^*(S))\Bigr] \notag \\
&\leq (1 - \alpha(|S|)) (g(S) + h(S)).
\end{align*}

\end{lemma}

\subsection{Greedy Algorithm and Approximation}

Finally, we can greedily pick arms one by one to maximize the incremental gain in \(f_{\text{upper}}(S)\), subject to budget \(I\). By standard results on submodular maximization with cardinality constraints \citep{Iyer2013}, this yields a \((1-1/e)\)-approximation for maximizing \(f_{\text{upper}}\). Moreover, we can combine it with the overhead term \((1-\alpha(|S|))\) to trade off between total payoff and probing cost.

Algorithm~\ref{alg:offline-greedy} first initializes $I{+}1$ empty probing sets \( S_0, S_1, \dots, S_I \) (lines 1–2). In each iteration \( i \) (lines 3–5), it selects the arm \( a \) that maximizes the marginal gain \( f_{\mathrm{upper}}(S_{i-1} \cup \{a\}) - f_{\mathrm{upper}}(S_{i-1}) \) and updates \( S_i \). The candidate sets \( S_j \) are then sorted by the adjusted reward \( (1-\alpha(|S_j|)) f_{\mathrm{upper}}(S_j) \) (line 7). The algorithm iterates through these sets (lines 8–14), returning an empty set if the highest-ranked \( S_j \) does not exceed \( h(\emptyset) \) (lines 9–10). Otherwise, it computes the expected reward \( \mathcal{R}(S_j) \) (line 11) and proceeds to the next candidate if \( (1-\alpha(|S_j|)) f_{\mathrm{upper}}(S_j) \) exceeds \( \zeta \mathcal{R}(S_j) \) (lines 11–12). If all conditions are met, \( S_j \) is returned as the final probing set $S^{\mathrm{pr}}$ (lines 14).

\begin{algorithm}[ht]
\caption{Offline Greedy Probing}
\label{alg:offline-greedy}
\begin{algorithmic}[1]

\Statex \textbf{Input:} $\{F_{m,a}\}_{m\in[M],\,a\in[A]},\;\alpha(\cdot),\;I,\;\zeta \geq 1.$
\Statex \textbf{Output:} $S^{\mathrm{pr}}$

\For{$i = 0$ to $I$} 
    \State $S_i \gets \emptyset$ 
\EndFor

\For{$i = 1$ to $I$} 
    \State $a \gets \displaystyle \arg\max\limits_{a \in [A]\setminus S_{i-1}} \Biggl[ \begin{aligned} & f_{\mathrm{upper}}(S_{i-1} \cup \{a\}) \\ & - f_{\mathrm{upper}}(S_{i-1}) \end{aligned} \Biggr]$
    \State $S_i \gets S_{i-1} \cup \{a\}$
\EndFor

\State $\Pi \gets \{ 0,1,\dots,I \}$
\State \textbf{sort} $\Pi$ so that if $i$ precedes $j$, then $(1-\alpha(|S_i|))\,f_{\mathrm{upper}}(S_i) \geq (1-\alpha(|S_j|))\,f_{\mathrm{upper}}(S_j)$

\For{each $j$ in $\Pi$ (largest to smallest upper-bound)}  
    \If{$(1-\alpha(|S_j|))\,f_{\mathrm{upper}}(S_j) < h(\emptyset)$}
        \State $S^{\mathrm{pr}} \leftarrow \emptyset$; \Return $S^{\mathrm{pr}}$
    \EndIf
    \If{$(1-\alpha(|S_j|))\,f_{\mathrm{upper}}(S_j) > \zeta \mathcal{R}(S_j)$}
        \State \textbf{continue}
    \Else
        \State $S^{\mathrm{pr}} \leftarrow S_j$; \Return $S^{\mathrm{pr}}$
    \EndIf
\EndFor

%\State \Return $\emptyset$
\end{algorithmic}
\end{algorithm}

\ignore{
\begin{theorem}\label{thm:offline-approx}
Let $S^*$ be an optimal subset maximizing $\mathcal{R}(S)$.
Then the set $S^{\mathrm{pr}}$ returned by Algorithm~\ref{alg:offline-greedy}, for any $\zeta \geq 1$, satisfies
\[
\mathcal{R}(S^{\mathrm{pr}}) 
\;\;\ge\;
\frac{\,e-1\,}{\,2e-1\,}\; \frac{1}{\zeta} 
\mathcal{R}(S^*).
\]
\end{theorem}
}

\begin{theorem}\label{thm:offline-approx}
Let \( S^* \) be an optimal subset maximizing \( \mathcal{R}(S) \).
Let
\[
j^\star \in \arg\max_{i\in\{0,\ldots,I\}}
\left\{(1-\alpha(|S_i|)) f_{\mathrm{upper}}(S_i)\right\}.
\]
Define
\[
\zeta_0 \;:=\;
\frac{(1-\alpha(|S_{j^\star}|))\, f_{\mathrm{upper}}(S_{j^\star})}{\mathcal{R}(S_{j^\star})},
\]
and assume \(\zeta \ge \max\{1,\zeta_0\}\).
Then the set \( S^{\mathrm{pr}} \) returned by Algorithm~\ref{alg:offline-greedy} satisfies
\[
\mathcal{R}(S^{\mathrm{pr}})
\;\ge\;
\frac{e-1}{2e-1}\cdot\frac{1}{\zeta}\,\mathcal{R}(S^*).
\]
\end{theorem}

Theorem~\ref{thm:offline-approx} provides the main theoretical guarantee of our work, with a detailed proof in the Appendix. Our algorithm effectively balances exploring additional arms and probing costs, ensuring a near-optimal reward in the offline setting.

\section{The Online Setting}

In the online setting, we consider a system with \(M\) agents and \(A\) arms over \(T\) rounds. Unlike the offline setting where rewards are known, here they are unknown, requiring a balance between exploration and exploitation.

The Online Fair Multi-Agent UCB with Probing (OFMUP) algorithm (Algorithm~\ref{alg:online-fair-ucb}) maintains empirical statistics for each agent--arm pair \((j,a)\), including an empirical CDF estimate \(\widehat{F}_{j,a}\), and constructs upper confidence bounds (UCBs). Our \textbf{key contribution} in the online setting is integrating Algorithm~\ref{alg:offline-greedy} into the online framework and designing a \textbf{novel UCB-based} strategy. This enables efficient learning while ensuring fairness across agents. The procedure executes the following steps:

\noindent\textbf{1. Initialization (Lines~1--4).}  
Each agent--arm pair $(j,a)$ starts with: $\widehat{F}_{j,a,t} \gets 1$, $N_{j,a,t} \gets 0$, $\hat{\mu}_{j,a,t} \gets 0$, $w_{j,a,t} \gets 0$.
These serve as optimistic estimates before data collection. The confidence bound \(U_{j,a,t}\) is defined later.

\noindent\textbf{2. Warm-Start Rounds (Lines~5--10).}  
For the first \(M A\) rounds, each agent--arm pair is explored at least once under assignment constraints. The selection follows:
\[
(\mathbf{e}_t)_{j,a} = 
\begin{cases} 
1, & \text{if } j = j_t \text{ and } a = a_t, \\
0, & \text{otherwise},
\end{cases} 
\]
This ensures each agent samples all arms and each arm is probed multiple times. 

\noindent\textbf{3. Main Loop (Lines~11--16).}  
For \(t > M A\), the algorithm iterates as follows:

\noindent\textbf{a. Probe Set Selection (Line~12).}  
\[
S_t \gets \textsc{Algorithm 1}\bigl(\widehat{F}_{j,a,t},\,\alpha(\cdot),\,I, \zeta\bigr),
\]
where \(I\) is the probing budget and \(\alpha(\cdot)\) the overhead function. This subroutine greedily selects \(S_t \subseteq [A]\) based on \(\widehat{F}_{j,a,t}\).

\noindent\textbf{b. Probing and Updates (Line~13).}  
Each arm in \(S_t\) is probed, revealing rewards \(R_{j,a,t}\), and updating \(\widehat{F}_{j,a,t}\), \(N_{j,a,t}\), \(\hat{\mu}_{j,a,t}\), and \(w_{j,a,t}\). The confidence bound is:
\[
U_{j,a,t} = \min(\hat{\mu}_{j,a,t} + w_{j,a,t}, 1).
\]

\noindent\textbf{c. Policy Optimization (Line~14).}  
The optimal policy is:
\[
\pi_t \gets \arg\max_{\pi_t \in \Delta^A} (1 - \alpha(\lvert S_t \rvert)) \cdot \mathbb{E}_{R_t}[\mathrm{NSW}(S_t, R_t, U_t, \pi_t)],
\]
where \(\mathrm{NSW}(\cdot)\) is the Nash social welfare objective.

\noindent\textbf{d. Arm Pulls and Final Updates (Lines~15--16).}  
Each agent \(j\) pulls \(a_{j,t} \sim \pi_t\), observes \(R_{j,a_{j,t},t}\), and updates \(\widehat{F}_{j,a_{j,t}}\), \(N_{j,a_{j,t}}\), \(\hat{\mu}_{j,a_{j,t}}\), \(w_{j,a_{j,t}}\), and \(U_{j,a_{j,t}}\).

\begin{algorithm}[!t]
\caption{Online Fair Multi-Agent UCB with Probing (OFMUP)}
\label{alg:online-fair-ucb}
\textbf{Input:} $A,M,T,I,c,\alpha(\cdot),\Delta$ etc.
\begin{algorithmic}[1]
\State \textbf{Initialize}:
\For{$j=1\to M,\;a=1\to A$}
   \State $\widehat{F}_{j,a,t}\gets 1,\quad N_{j,a,t}\gets0,\quad \hat{\mu}_{j,a,t}\gets0,$ 
   \State $w_{j,a,t}\gets0$
\EndFor

\For{$t=1\to M A$} 
   \State $j_t \gets ((t-1) \mod M) + 1$ 
   \State $a_t \gets ((t-1) / M) + 1$ 
   \State $\pi_t \gets \mathbf{e}_t,\quad S_t\gets\{a_t\}$
   \State \textbf{agent } $j_t$ \textbf{pulls arm } $a_t$
   \State \textbf{observe }$R_{j_t,a_t,t}$, \textbf{update }$\widehat{F}_{j_t,a_t,t},\,\hat{\mu}_{j_t,a_t,t},$
   \Statex\hspace{\algorithmicindent} $w_{j_t,a_t,t},\,N_{j_t,a_t,t}$
\EndFor

\For{$t=M A+1\to T$}
   \State $S_t \;\gets\;\textsc{Algorithm 1}\bigl(\widehat{F}_{j,a,t},\alpha(\cdot),I,\zeta\bigr)$
   \State \textbf{probe each arm in } $S_t$, \textbf{observe } $R_{j,a,t}$, 
   \Statex\hspace{\algorithmicindent} \textbf{update }$\widehat{F}_{j,a,t},N_{j,a,t},\hat{\mu}_{j,a,t},\,w_{j,a,t},$
   \Statex\hspace{\algorithmicindent} $ U_{j,a,t}=\min(\hat{\mu}_{j,a,t}+w_{j,a,t},1)$
    \State $\pi_t \gets \arg\max\limits_{\pi_t \in \Delta^A} \Biggl[
    \begin{aligned}
        &(1-\alpha(|S_t|)) \cdot \\
        &\mathbb{E}_{R_t} \bigl[\mathrm{NSW}(S_t, R_t, U_t, \pi_t)\bigr]
    \end{aligned}
    \Biggr]$
   \State \textbf{each agent }$j$ \textbf{pulls } $a_{j,t}\!\sim\!\pi_t$, \textbf{observe }$R_{j,a_{j,t},t}$
\State \textbf{update:}
\Statex \hspace{\algorithmicindent}%
  $\widehat{F}_{j,a_{j,t},t},\,
   N_{j,a_{j,t},t},\,
   \hat{\mu}_{j,a_{j,t},t},\,
   w_{j,a_{j,t},t},\,
   U_{j,a_{j,t},t}$

\EndFor
\end{algorithmic}
\end{algorithm}

\section*{Analysis}

In this section, we present our regret analysis for the proposed %the Online Fair Multi-Agent UCB with Probing (OFMUP) 
OFMUP algorithm. Detailed proofs and Lemma 8 are deferred to the Appendix.

\subsection*{Smoothness of the NSW Objective}

%A key ingredient in our analysis is the smoothness of the NSW objective. In our setting the NSW function accounts for both the observed rewards and the baseline means. We have the following result.

%A key ingredient in our analysis is the smoothness of the NSW objective. We have the following result.

\begin{lemma}[Smoothness of the NSW Objective]\label{lem:nsw_smooth}
%Let \(\pi \in \Delta^{M \times A}\) be any assignment matrix where \(\pi_{j,a}\) denotes the probability that agent \(j\) is assigned arm \(a\). 
Let \(\mu,\nu \in [0,1]^{M \times A}\) be two reward matrices. For any probing set \(S\subseteq [A]\) and observed rewards \(R\), we have
\[
\begin{aligned}
\Bigl|\mathrm{NSW}(S,R,\mu,\pi) &- \mathrm{NSW}(S,R,\nu,\pi)\Bigr|\\[1mm]
&\le \sum_{j=1}^{M}\sum_{a=1}^{A} \pi_{j,a}\,\Bigl|\mu_{j,a} - \nu_{j,a}\Bigr|.
\end{aligned}
\]
\end{lemma}

\subsection*{Concentration of Reward Estimates}

\begin{lemma}[Concentration of Reward Estimates]\label{lem:4.2prime}
Let \(\delta \in (0,1)\). Then with probability at least \(1-\frac{\delta}{2}\), for all \(t > A\), \(a \in [A]\), and \(j \in [M]\), we have
\[
\begin{aligned}
\Bigl|\mu_{j,a} - \hat{\mu}_{j,a,t}\Bigr|
&\le \sqrt{\frac{2(\hat{\mu}_{j,a,t} - \hat{\mu}_{j,a,t}^2) \ln\!\bigl(\tfrac{2 M A T}{\delta}\bigr)}{N_{j,a,t}}}\\[1mm]
&\quad + \frac{\ln\!\Bigl(\frac{2\,M\,A\,T}{\delta}\Bigr)}{3\,N_{j,a,t}}
= w_{j,a,t}\
\end{aligned}
\]
\end{lemma}

\ignore{
\subsection*{Aggregate Reward Bound}

Next, we introduce a bound relating the aggregated expected reward (computed with true means) to the confidence widths. \begin{lemma}[Aggregate Reward Bound]\label{lem:aggreward}
Let
\begin{align*}
\gamma_{j,t} \;=\; \sum_{a \in [A]} \,\pi_{j,a,t}\,\bigl(1 - U_{j,a,t}\bigr),\\[1mm]
Q(t,p) \;=\; \bigl\{\,j \in [M] : \gamma_{j,t} \,\ge\, 2^{-p}\bigr\}\,.
\end{align*}
Suppose that for all integers \(p \ge 0\), we have
\[
|Q(t,p)| < 2^p \cdot 3\,\ln T.
\]
Then the following inequality holds:
\[
\begin{aligned}
\sum_{j=1}^{M} \sum_{a=1}^{A}
\pi_{j,a,t}\left(\hat{\mu}_{j,a,t} - \hat{\mu}_{j,a,t}^2\right) 
&\leq 1 + 6\ln T\log M \\
&+ \sum_{j=1}^{M} \sum_{a=1}^{A}
\pi_{j,a,t}\hat{\mu}_{j,a,t}w_{j,a,t}.
\end{aligned}
\]
\end{lemma}

This lemma is crucial in controlling the error introduced by using upper confidence bounds.
}

\begin{figure*}[t]
 	\centering
 		\subfloat[]{%
 		\includegraphics[width=0.245\textwidth]
                {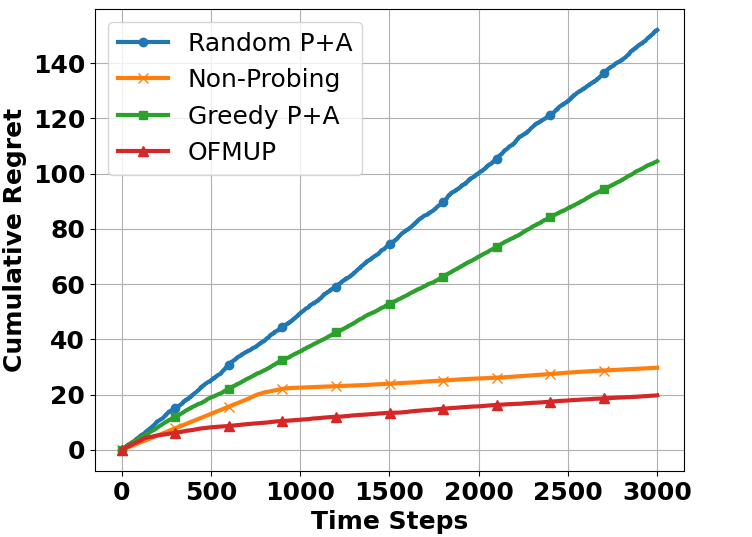}%
                \label{fig1}
 	  }
 	      \subfloat[]{%
 		\includegraphics[width=0.245\textwidth]
                {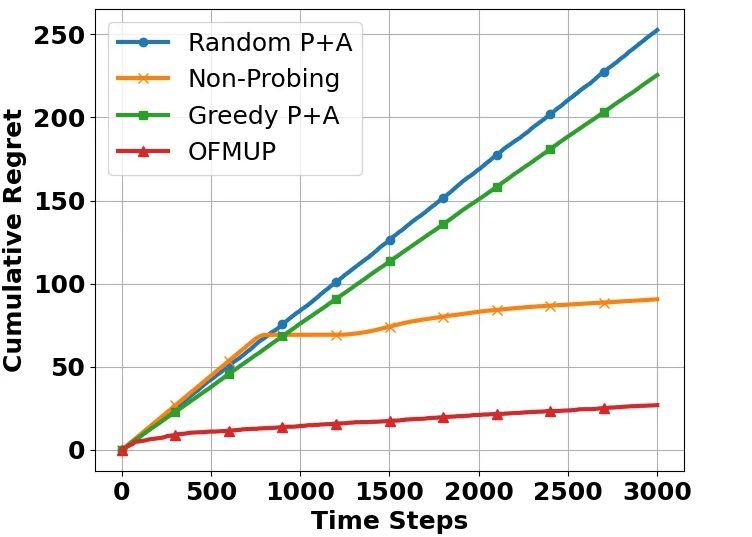}%
                \label{fig2}
 	}
 	      \subfloat[]{%
 		\includegraphics[width=0.245\textwidth]{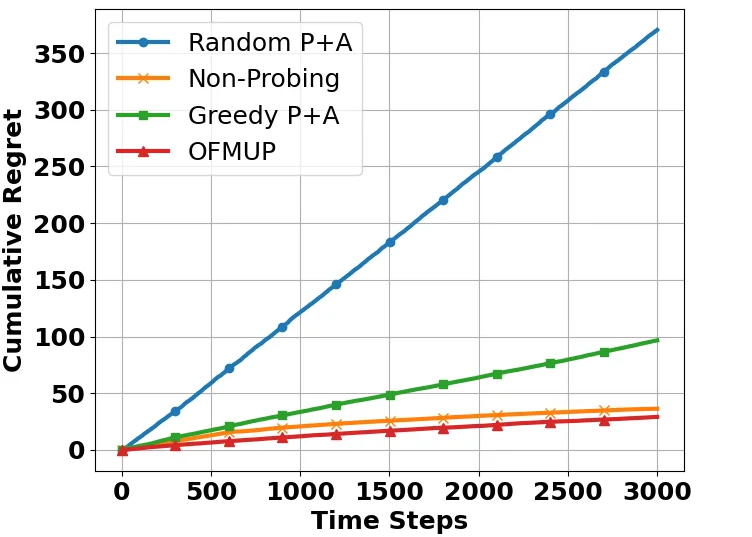}%
                \label{fig3}
 	}
 	      \subfloat[]{%
 		\includegraphics[width=0.245\textwidth]{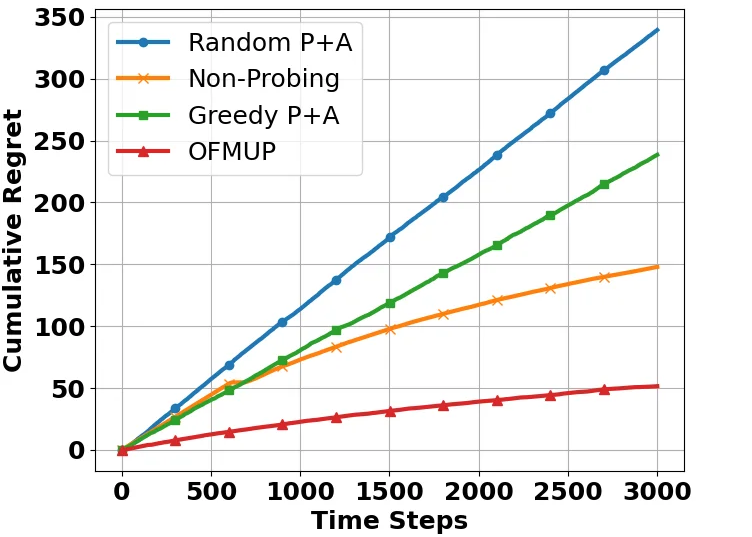}%
            \label{fig4}
 	}

 	 \caption{\small (a): Agents number $M=12$, arms number $A=8$, Bernoulli distribution for reward. 
     (b): Agents number $M=20$, arms number $A=10$, Bernoulli distribution for reward. 
     (c): Agents number $M=12$, arms number $A=8$, General distribution for reward. 
     (d): Agents number $M=20$, arms number $A=10$, General distribution for reward. Data from NYYellowTaxi 2016.}
 	\label{fig:ablation}
 \end{figure*}

\subsection*{Main Regret Guarantee}

Combining the smoothness of the NSW objective, the concentration bounds, and the auxiliary technical results, we obtain our main regret guarantee.

\begin{theorem}\label{thm:regret}
For any $\delta \in (0,1)$, with probability at least $1 - \delta$, the cumulative regret $\mathcal{R}_{\mathrm{regret}}(T)$ of the Online Fair multi-Agent UCB with Probing algorithm (Algorithm~2) satisfies
\[
  \mathcal{R}_{\mathrm{regret}}(T) \;=\; O\!\Bigl(\,\zeta\;\bigl(\sqrt{M\,A\,T} \;+\; M\,A\bigr)\,
               \ln^{c}\!\Bigl(\tfrac{M\,A\,T}{\delta}\Bigr)\Bigr),
\]
for some constant $c > 0$.
\end{theorem}

Theorem~\ref{thm:regret} shows that our probing-based algorithm achieves sublinear regret with a provable constant-factor improvement over the non-probing baseline. This gain arises from faster elimination of suboptimal arms and more informed assignment decisions. Experimental results confirm that probing consistently lowers regret across all time horizons. Full proofs of Lemmas~\ref{lem:nsw_smooth}–8 and Theorem~\ref{thm:regret} are provided in the Appendix.

%Theorem~\ref{thm:regret} establishes that the cumulative regret of our probing-based algorithm is sublinear in $T$. Hence, as $T$ grows, the average regret per round tends to zero. Although the probing mechanism does not alter the $\sqrt{T}$ growth rate in the worst case, it substantially improves empirical performance by enabling faster elimination of suboptimal arms and more informed decision-making, as confirmed by our experimental results. Full proofs of Lemmas~\ref{lem:nsw_smooth}--8
%\ref{lem:aggreward} and Theorem~\ref{thm:regret} are provided in the Appendix.

\section{Experiments}

We evaluate our framework using controlled simulations and a real-world ridesharing case study. %with results demonstrating that the probing mechanism yields superior performance. 
In our simulated environment, we consider a multi-agent multi-armed bandit (MA-MAB) setting with \(M\) agents and \(A\) arms (e.g., \(M=12, A=8\) for small-scale and \(M=20, A=10\) for large-scale scenarios). For each agent–arm pair \((j,a)\), rewards are generated as i.i.d. samples from a fixed distribution \(D_{j,a}\) with mean \(\mu_{j,a}\). In the real-world setting, we apply our framework to the NYYellowTaxi 2016 dataset~\citep{shah2020neural}, treating vehicles as agents and discretized pickup locations (binned into 0.01° grids) as arms. Rewards are determined by the normalized Manhattan distance between vehicles and pickup points—closer distances yield higher rewards. Vehicle locations are randomly pre-sampled within city bounds and remain fixed, underscoring the practical effectiveness of our approach.
 \begin{figure}[H]
    \centering
    \subfloat[]{%
        \includegraphics[width=0.24\textwidth]{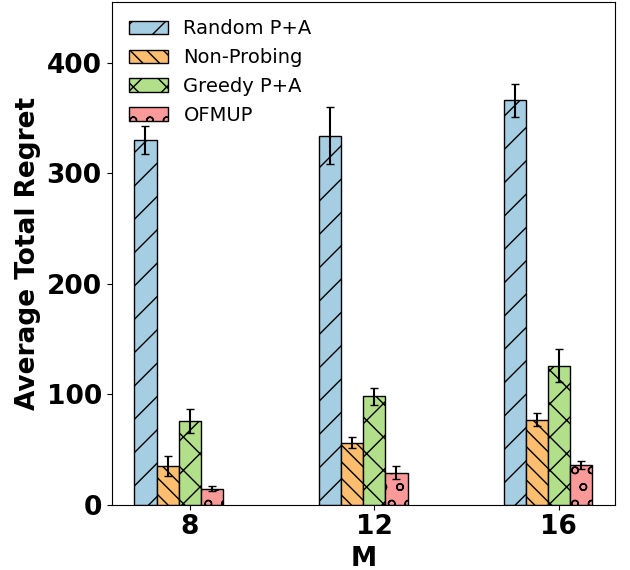}%
        \label{fig:scalability_agents}
    }
    \subfloat[]{%
        \includegraphics[width=0.24\textwidth]{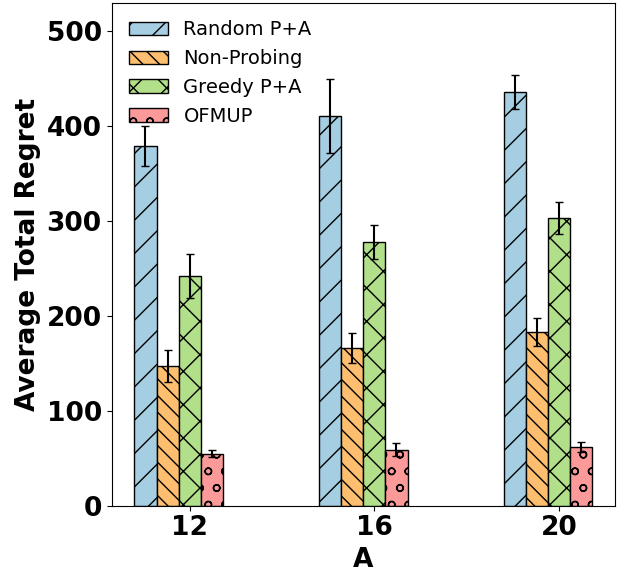}%
        \label{fig:scalability_arms}
    }
    \caption{\small Scalability analysis across two dimensions: (a) Fixed arms number $A=8$ with varying agents; (b) Fixed agents number $M=20$ with varying arms.}
    \label{fig:scalability}
\end{figure}

To test different conditions, we consider two reward distributions: a Bernoulli distribution (with rewards 0 or 1 and with mean rewards \(\mu\) in \([0.3,0.8]\)) and a discrete distribution (with rewards sampled from \(\{0.3,0.4,0.5,0.6,0.7,0.8\}\)). Cumulative regret is computed as in Equation~(\ref{regret}), with the optimal reward determined via exhaustive search. For numerical stability, cumulative Nash social welfare is aggregated using the geometric mean of per-agent rewards. We further verify that our offline objective \(f_{\mathrm{upper}}(S)\) closely approximates \(\log(g(S))\) (difference 0.03), confirming the validity of our formulation. All experimental settings satisfy the assumptions outlined in our paper. \footnote{Implementation available at:  https://github.com/jiaxin26/Fair-MA-MAB-with-Probing}

\subsection{Baselines}
Three algorithms serve as our comparison baselines: \textbf{Non-Probing}, a fair MAB algorithm from~\citet{jones2023efficient} without probing capability, focusing only on optimal assignment with current information; \textbf{Random Probing with Random assignment (Random P+A)}, which randomly selects a fixed number of arms for probing and then assigns the arms randomly; and \textbf{Greedy Probing with Random assignment (Greedy P+A)}, which uses the same greedy probing strategy as our algorithm but performs random assignment after probing.

\subsection{Results}
Figure 1 shows OFMUP's performance across different tests. In small-scale tests (M = 12, A = 8), OFMUP reduces regret by 85\% vs. Random P+A and 60\% vs. Greedy P+A after 3000 steps. The advantage is greater in medium-scale tests (M = 20, A = 10), where regret is 88\% lower than Random P+A and 80\% lower than Greedy P+A at T=3000, showing improved scalability.
For discrete rewards, OFMUP continues to outperform. The gap with other methods grows as it learns reward patterns, reducing regret by 85\% vs. Random P+A and 65\% vs. Non-Probing at T=3000. Notably, Random P+A's regret is slightly higher in small-scale tests due to fairness computation via the geometric mean and increased variability with smaller samples. These results validate our theoretical analysis—OFMUP effectively balances exploration and exploitation while ensuring fairness. The experiments further confirm the crucial role of probing in gathering information and guiding assignment. We further examine scalability by independently varying agent numbers (Figure 2). OFMUP performs well in all our tests, and its advantage over the other methods grows as the problems become more complex.

\section{Conclusion}
%We introduce a fair multi-agent multi-armed bandit framework that integrates selective probing with Nash Social Welfare objectives, enabling efficient information gathering and balanced decision-making. Our offline greedy probing algorithm provides theoretical guarantees, while the online extension ensures sublinear regret. Experiments on synthetic and real-world ridesharing data validate its effectiveness, demonstrating superior performance over baselines. The results confirm that probing refines reward estimates and improves the trade-off between exploration and exploitation, making our framework suitable for practical resource assignment.

We propose a fair MA-MAB framework combining selective probing with Nash Social Welfare objectives for efficient information gathering and equitable assignments. Our offline greedy algorithm attains a constant‐factor approximation bound and the online extension guarantees sublinear regret. Experiments on synthetic and real‐world ridesharing data show refined reward estimates, improved exploration–exploitation balance, and superior performance over baselines, demonstrating practical applicability.

\section{Acknowledgments}
Mattei was supported in part by NSF Awards IIS-III-2107505, IIS-RI-2134857, IIS-RI-2339880, and CNS-SCC-2427237, as well as the Harold L. and Heather E. Jurist Center of Excellence for Artificial Intelligence at Tulane University and the Tulane University Center of Excellence for Community-Engaged Artificial Intelligence (CEAI). Zheng was supported in part by NSF Award CNS-2146548 and Louisiana BOR-RCS award 088A-25.

\bibliography{aaai2026}

@inproceedings{zuo2022hierarchical,
  title={Hierarchical conversational preference elicitation with bandit feedback},
  author={Zuo, Jinhang and Hu, Songwen and Yu, Tong and Li, Shuai and Zhao, Handong and Joe-Wong, Carlee},
  booktitle={Proceedings of the 31st ACM International Conference on Information \& Knowledge Management},
  pages={2827--2836},
  year={2022}
}

@inproceedings{zhang2020conversational,
author = {Zhang, Xiaoying and Xie, Hong and Li, Hang and C.S. Lui, John},
title = {Conversational Contextual Bandit: Algorithm and Application},
year = {2020},
booktitle = {Proceedings of The Web Conference 2020},
pages = {662–672},
}

@inproceedings{NEURIPS2021_c96ebeee,
  author    = {Hossain, S. and Micha, E. and Shah, N.},
  year      = {2021},
  title     = {Fair Algorithms for Multi-Agent Multi-Armed Bandits},
  booktitle = {Proceedings of Advances in Neural Information Processing Systems},
  pages     = {24005-24017}
}

@inproceedings{NIPS2016_eb163727,
  author    = {Joseph, M. and Kearns, M. and Morgenstern, J. H. and Roth, A.},
  year      = {2016},
  title     = {Fairness in Learning: Classic and Contextual Bandits},
  booktitle = {Proceedings of Advances in Neural Information Processing Systems}
}

@article{pandp.20181018,
  author  = {Kleinberg, J. and Ludwig, J. and Mullainathan, S. and Rambachan, A.},
  year    = {2018},
  title   = {Algorithmic Fairness},
  journal = {AEA Papers and Proceedings}
}

@inproceedings{pmlr-v32-agarwalb14,
  author    = {Agarwal, A. and Hsu, D. and Kale, S. and Langford, J. and Li, L. and Schapire, R.},
  year      = {2014},
  title     = {Taming the Monster: A Fast and Simple Algorithm for Contextual Bandits},
  booktitle = {Proceedings of the 31st International Conference on Machine Learning},
  pages     = {1638-1646}
}

@article{Zuo_Zhang_Joe-Wong_2020,
  author  = {Zuo, J. and Zhang, X. and Joe-Wong, C.},
  year    = {2020},
  title   = {Observe Before Play: Multi-Armed Bandit with Pre-Observations},
  journal = {Proceedings of the AAAI Conference on Artificial Intelligence},
  pages   = {7023-7030}
}

@article{abs-2108-03297,
  author  = {Xu, T. and Zhang, D. and Pathak, P. H. and Zheng, Z.},
  year    = {2021},
  title   = {Joint {AP} Probing and Scheduling: {A} Contextual Bandit Approach},
  journal = {IEEE Military Communications Conference (MILCOM)}
}

@inproceedings{Meritocratic2018,
author = {Joseph, Matthew and Kearns, Michael and Morgenstern, Jamie and Neel, Seth and Roth, Aaron},
title = {Meritocratic Fairness for Infinite and Contextual Bandits},
year = {2018},
booktitle = {Proceedings of the 2018 AAAI/ACM Conference on AI, Ethics, and Society},
pages = {158–163},
}

@inproceedings{NEURIPS2018_be3159ad,
 author = {Heidari, Hoda and Ferrari, Claudio and Gummadi, Krishna and Krause, Andreas},
 booktitle = {Advances in Neural Information Processing Systems},
 title = {Fairness Behind a Veil of Ignorance: A Welfare Analysis for Automated Decision Making},
 year = {2018}
}

@article{Zhang_Conitzer_2021, 
title={Incentive-Aware PAC Learning}, 
journal={Proceedings of the AAAI Conference on Artificial Intelligence}, author={Zhang, Hanrui and Conitzer, Vincent}, 
year={2021},
pages={5797-5804} 
}

@article{Patil_Ghalme_Nair_Narahari_2020,
  title={Achieving fairness in the stochastic multi-armed bandit problem},
  author={Patil, Vishakha and Ghalme, Ganesh and Nair, Vineet and Narahari, Yadati},
  journal={Journal of Machine Learning Research},
  volume={22},
  number={174},
  pages={1--31},
  year={2021}
}

@article{xu2023online,
  author  = {Xu, T. and Zhang, D. and Zheng, Z.},
  year={2023},
  title={Online Learning for Adaptive Probing and Scheduling in Dense WLANs},
  journal={IEEE INFOCOM 2023-IEEE Conference on Computer Communications},
}

@article{golovin_krause_2011,
  author  = {Golovin, D. and Krause, A.},
  year    = {2011},
  title   = {Adaptive Submodularity: Theory and Applications in Active Learning and Stochastic Optimization},
  journal = {Journal of Artificial Intelligence Research},
  pages   = {427-486}
}

@book{8895728,
  author = {Slivkins, A.},
  year   = {2019},
  title  = {Introduction to Multi-Armed Bandits},
  publisher = {Now Publishers Inc.}
}

@inproceedings{NEURIPS2019_36d75342,
 author = {Oh, Min-hwan and Iyengar, Garud},
 booktitle = {Advances in Neural Information Processing Systems},
 title = {Thompson Sampling for Multinomial Logit Contextual Bandits},
 year = {2019}
}

@inproceedings{garivier2013,
  title={The KL-UCB algorithm for bounded stochastic bandits and beyond},
  author={Garivier, Aur{\'e}lien and Capp{\'e}, Olivier},
  booktitle={Proceedings of the 24th annual conference on learning theory},
  pages={359--376},
  year={2011},
  organization={JMLR Workshop and Conference Proceedings}
}

@article{Chen_Javdani_2015, 
title={Submodular Surrogates for Value of Information}, 
journal={Proceedings of the AAAI Conference on Artificial Intelligence}, 
author={Chen, Yuxin and Javdani, Shervin and Karbasi, Amin and Bagnell, J. and Srinivasa, Siddhartha and Krause, Andreas}, 
year={2015}, 
}

@inproceedings{NIPS2014_66368270,
 author = {Amin, Kareem and Rostamizadeh, Afshin and Syed, Umar},
 booktitle = {Advances in Neural Information Processing Systems},
 title = {Repeated Contextual Auctions with Strategic Buyers},
 year = {2014}
}

@inproceedings{NEURIPS2019_85353d3b,
 author = {Mart\'{\i}nez-Rubio, David and Kanade, Varun and Rebeschini, Patrick},
 booktitle = {Advances in Neural Information Processing Systems},
 title = {Decentralized Cooperative Stochastic Bandits},
 year = {2019}
}

@article{LAI19854,
  author  = {Lai, T. L. and Robbins, H.},
  year    = {1985},
  title   = {Asymptotically efficient adaptive allocation rules},
  journal = {Advances in Applied Mathematics},
  pages   = {4-22}
}

@article{bubeck2012,
  title={Regret analysis of stochastic and nonstochastic multi-armed bandit problems},
  author={Bubeck, S{\'e}bastien and Cesa-Bianchi, Nicolo and others},
  journal={Foundations and Trends{\textregistered} in Machine Learning},
  volume={5},
  number={1},
  pages={1--122},
  year={2012},
  publisher={Now Publishers, Inc.}
}

@inproceedings{li2016collaborative,
author = {Li, Shuai and Karatzoglou, Alexandros and Gentile, Claudio},
title = {Collaborative Filtering Bandits},
year = {2016},
booktitle = {Proceedings of the 39th International ACM SIGIR Conference on Research and Development in Information Retrieval},
pages = {539–548},
}

@article{Gai2014DistributedSO,
  title={Distributed Stochastic Online Learning Policies for Opportunistic Spectrum Access},
  author={Yi Gai and Bhaskar Krishnamachari},
  journal={IEEE Transactions on Signal Processing},
  year={2014},
  pages={6184-6193},
}

@article{Liu_2010,
   title={Distributed Learning in Multi-Armed Bandit With Multiple Players},
   journal={IEEE Transactions on Signal Processing},
   author={Liu, Keqin and Zhao, Qing},
   year={2010},
   pages={5667–5681} 
}

@book{Lattimore_Szepesvári_2020,
  author = {Lattimore, T. and Szepesvári, C.},
  year   = {2020},
  title  = {Bandit Algorithms},
  publisher = {Cambridge University Press}
}

@article{Kaneko1979,
  author  = {Kaneko, M. and Nakamura, K.},
  year    = {1979},
  title   = {The Nash Social Welfare Function},
  journal = {Econometrica},
  pages   = {423-435}
}

@article{HorstTuyBook,
  title={Global optimization: Deterministic approaches},
  author={Salhi, A},
  journal={Journal of the Operational Research Society},
  volume={45},
  number={5},
  pages={595--597},
  year={1994},
  publisher={Taylor \& Francis}
}

@book{TawarmalaniSahinidisBook,
  title={Convexification and global optimization in continuous and mixed-integer nonlinear programming: theory, algorithms, software, and applications},
  author={Tawarmalani, Mohit and Sahinidis, Nikolaos V},
  volume={65},
  year={2013},
  publisher={Springer Science \& Business Media}
}

@article{Iyer2013,
  author  = {Iyer, R. K. and Bilmes, J. A.},
  year    = {2013},
  title   = {Submodular optimization with submodular cover and submodular knapsack constraints},
  journal = {Advances in neural information processing systems}
}

@inproceedings{jones2023efficient,
  author    = {Jones, M. and Nguyen, H. and Nguyen, T.},
  year      = {2023},
  title     = {An efficient algorithm for fair multi-agent multi-armed bandit with low regret},
  booktitle = {Proceedings of the AAAI Conference on Artificial Intelligence},
  pages     = {8159-8167}
}

@inproceedings{cole2015approximating,
  author    = {Cole, R. and Gkatzelis, V.},
  year      = {2015},
  title     = {Approximating the Nash social welfare with indivisible items},
  booktitle = {Proceedings of the Forty-Seventh Annual ACM Symposium on Theory of Computing},
  pages     = {371-380}
}

@inproceedings{shah2020neural,
  title={Neural approximate dynamic programming for on-demand ride-pooling},
  author={Shah, Sanket and Lowalekar, Meghna and Varakantham, Pradeep},
  booktitle={Proceedings of the AAAI Conference on Artificial Intelligence},
  year={2020}
}

@article{pandora-Weitzman,
  title={Optimal Search for the Best Alternative},
  author={Martin L. Weitzman},
  journal={Econometrica},
  volume={47},
  number={3},
  pages={641-654},
  year={1979}
}

@article{deshpande2016approximation,
  title={Approximation algorithms for stochastic submodular set cover with applications to boolean function evaluation and min-knapsack},
  author={Deshpande, Amol and Hellerstein, Lisa and Kletenik, Devorah},
  journal={ACM Transactions on Algorithms (TALG)},
  volume={12},
  number={3},
  pages={1--28},
  year={2016},
  publisher={ACM New York, NY, USA}
}

@inproceedings{liu2008near,
  title={Near-optimal algorithms for shared filter evaluation in data stream systems},
  author={Liu, Zhen and Parthasarathy, Srinivasan and Ranganathan, Anand and Yang, Hao},
  booktitle={ACM SIGMOD},
  year={2008}
}

@inproceedings{bhaskara2020adaptive,
  title={Adaptive probing policies for shortest path routing},
  author={Bhaskara, Aditya and Gollapudi, Sreenivas and Kollias, Kostas and Munagala, Kamesh},
  booktitle={Advances in Neural Information Processing Systems},
  year={2020}
}

@inproceedings{goel2006asking,
  title={Asking the right questions: Model-driven optimization using probes},
  author={Goel, Ashish and Guha, Sudipto and Munagala, Kamesh},
  booktitle={Proceedings of the twenty-fifth ACM SIGMOD-SIGACT-SIGART symposium on Principles of database systems},
  pages={203--212},
  year={2006}
}

@article{AhmedAtamturk2017,
  author  = {Shabbir Ahmed and Alper Atamt{\"u}rk},
  title   = {Maximizing a Class of Submodular Utility Functions},
  journal = {Mathematical Programming},
  volume  = {128},
  number  = {1},
  pages   = {149--169},
  year    = {2011},
}

@inproceedings{tucker2023costlyobs,
  author    = {Aaron D. Tucker and Caleb Biddulph and Claire Wang and Thorsten Joachims},
  title     = {Bandits with Costly Reward Observations},
  booktitle = {Proceedings of the 39th Conference on Uncertainty in Artificial Intelligence (UAI)},
  pages     = {2147--2156},
  year      = {2023},
  publisher = {PMLR}
}

@article{elumar2024costlyprobe,
  author  = {Eray Can Elumar and Cem Tekin and Osman Ya\u{g}an},
  title   = {Multi-{Armed} Bandits with Costly Probes},
  journal = {IEEE Transactions on Information Theory},
  year    = {2024},
  note    = {Early Access},
  doi     = {10.1109/TIT.2024.3506866}
}

@article{caragiannis2019,
  title={The unreasonable fairness of maximum Nash welfare},
  author={Caragiannis, Ioannis and Kurokawa, David and Moulin, Herv{\'e} and Procaccia, Ariel D and Shah, Nisarg and Wang, Junxing},
  journal={ACM Transactions on Economics and Computation (TEAC)},
  volume={7},
  number={3},
  pages={1--32},
  year={2019},
  publisher={ACM New York, NY, USA}
}

@article{abdollahpouri2020multistakeholder,
  title={Multistakeholder recommendation: Survey and research directions},
  author={Abdollahpouri, Himan and Adomavicius, Gediminas and Burke, Robin and Guy, Ido and Jannach, Dietmar and Kamishima, Toshihiro and Krasnodebski, Jan and Pizzato, Luiz},
  journal={User Modeling and User-Adapted Interaction},
  volume={30},
  number={1},
  pages={127--158},
  year={2020},
  publisher={Springer}
}

@INPROCEEDINGS{892067,
  author={Kumar, A. and Kleinberg, J.},
  booktitle={Proceedings 41st Annual Symposium on Foundations of Computer Science}, 
  title={Fairness measures for resource allocation}, 
  year={2000},
  number={},
  pages={75-85},
}

@incollection{THOMSON2011393,
title = {Chapter Twenty-One - Fair Allocation Rules},
series = {Handbook of Social Choice and Welfare},
publisher = {Elsevier},
volume = {2},
pages = {393-506},
year = {2011},
booktitle = {Handbook of Social Choice and Welfare},
author = {William Thomson},
}

@article{xu2025online,
  title={Online Learning with Probing for Sequential User-Centric Selection},
  author={Xu, Tianyi and Chen, Yiting and Li, Henger and Bian, Zheyong and Dall'Anese, Emiliano and Zheng, Zizhan},
  journal={arXiv preprint arXiv:2507.20112},
  year={2025}
}

\newpage

\onecolumn

\title{Fair Algorithms with Probing for Multi-Agent Multi-Armed Bandits\\(Supplementary Material)}
\maketitle

\appendix

\section*{Appendix}

\setcounter{lemma}{0}
\setcounter{theorem}{0}

\subsection{Offline Setting}
%\subsection{Submodularity Properties}

\begin{lemma}[Monotonicity of \(g(S)\)]\label{lemma:monotonicity_g-appendix}
For any \(S \subseteq T \subseteq [A]\), we have
\[
g(S) \le g(T).
\]
\end{lemma}

\begin{proof}
Recall that
\[
g(S) = \max_{\pi \in \Delta^A_S} \;\mathbb{E}\!\left[
\prod_{j \in [M]} \left(\sum_{a \in S} \pi_{j,a,t} \, R_{j,a}\right)
\right],
\]
where 
\[
\Delta^A_S = \Bigl\{ \pi \in \mathbb{R}_+^{M \times A} :\; \pi_{j,a,t} = 0 \text{ for all } a \notin S,\quad \sum_{a \in S} \pi_{j,a,t} \le 1 \text{ for all } j \in [M] \Bigr\}.
\]
If \(S \subseteq T\), then any allocation \(\pi \in \Delta^A_S\) satisfies \(\pi_{j,a,t} = 0\) for all \(a \in T\setminus S\). Consequently, for each \(j \in [M]\) we have
\[
\sum_{a \in T} \pi_{j,a,t} = \sum_{a \in S} \pi_{j,a,t} \le 1.
\]
That is, \(\pi\) is also a feasible solution in \(\Delta^A_T\). Since the maximization defining \(g(T)\) is taken over a larger set than that for \(g(S)\), we immediately obtain
\[
g(S) = \max_{\pi \in \Delta^A_S} \;\mathbb{E}\!\left[
\prod_{j \in [M]} \left(\sum_{a \in S} \pi_{j,a,t} \, R_{j,a}\right)
\right]
\le \max_{\pi \in \Delta^A_T} \;\mathbb{E}\!\left[
\prod_{j \in [M]} \left(\sum_{a \in T} \pi_{j,a,t} \, R_{j,a}\right)
\right]
= g(T).
\]
This completes the proof.
\end{proof}

\begin{lemma}[Monotonicity]
For any \(S \subseteq T \subseteq [A]\), we have \( f_{\text{upper}}(S) \le f_{\text{upper}}(T) \).
\end{lemma}
\begin{proof}
Since \(g(S)\le g(T)\) (Lemma \ref{lemma:monotonicity_g}), and \(\phi(\cdot)\) is non-decreasing, we get
\[
f_{\text{upper}}(S)=\phi\bigl(g(S)\bigr)\le\phi\bigl(g(T)\bigr)=f_{\text{upper}}(T).
\]
\end{proof}

\begin{lemma}[Submodularity]
For any \( S \subseteq T \subseteq [A] \) and \( a \notin T \),
\[
f_{\text{upper}}(S \cup \{a\}) - f_{\text{upper}}(S)
\;\;\ge\;\;
f_{\text{upper}}(T \cup \{a\}) - f_{\text{upper}}(T).
\]
\end{lemma}

\begin{proof}
Let \(S \subseteq T \subseteq [A]\) and \(a \in [A]\setminus T\). Define
\[
x = g(S),\quad y = g(T),\quad d = g(S\cup\{a\}) - g(S),\quad d' = g(T\cup\{a\}) - g(T).
\]
By the monotonicity of \(g\), we have \(x \le y\).

Based on the assumption made in the offline analysis that the partition \(\{\tau_i\}\) used in constructing \(\phi\) is sufficiently fine so that for any \(S\) and any \(a \notin S\) the marginal gain satisfies
\[
g(S\cup\{a\}) - g(S) \le \eta,
\]
for some small \(\eta > 0\), it follows that \(x\) and \(x+d\) lie in some interval \([\tau_i,\tau_{i+1}]\) and \(y\) and \(y+d'\) lie in some interval \([\tau_j,\tau_{j+1}]\). Based on the construction of \(\phi\), on the interval \([\tau_i,\tau_{i+1}]\) we have
\[
\phi(z) = T_{\tau_i}(z) = \log(\tau_i) + \frac{1}{\tau_i}(z-\tau_i),
\]
so that
\[
\phi(x+d)-\phi(x) = \frac{1}{\tau_i}\,d.
\]
Similarly, on \([\tau_j,\tau_{j+1}]\),
\[
\phi(y+d')-\phi(y) = \frac{1}{\tau_j}\,d'.
\]
Since \(x \le y\) and by monotonicity of \(g\) the corresponding breakpoint indices satisfy \(\tau_i \le \tau_j\), it follows that
\[
\frac{1}{\tau_i} \ge \frac{1}{\tau_j}.
\]

In the case \(d \ge d'\), the desired inequality immediately holds:
\[
\phi(x+d)-\phi(x) = \frac{1}{\tau_i}\,d \ge \frac{1}{\tau_j}\,d' = \phi(y+d')-\phi(y).
\]
In the case \(d < d'\), we further rely on the additional assumption that the partition is chosen sufficiently fine so that even if the absolute increment \(d\) is smaller than \(d'\), the relative difference in slopes compensates. Specifically, we assume that
\[
\frac{d}{d'} \ge \frac{\tau_i}{\tau_j}.
\]
Under this condition, we have
\[
\frac{1}{\tau_i}\,d \ge \frac{1}{\tau_i}\Bigl(\frac{\tau_i}{\tau_j}d'\Bigr) = \frac{1}{\tau_j}\,d',
\]
so that
\[
\phi(x+d)-\phi(x) \ge \phi(y+d')-\phi(y).
\]

Thus, in both cases, we conclude that
\[
\phi\bigl(g(S\cup\{a\})\bigr) - \phi\bigl(g(S)\bigr) \ge \phi\bigl(g(T\cup\{a\})\bigr) - \phi\bigl(g(T)\bigr).
\]
This shows that \(f_{\mathrm{upper}}(S)=\phi(g(S))\) is submodular.
\end{proof}

\begin{lemma}[Monotonicity of \(h(S)\)]\label{lemma:monotonicity_h-appendix}
For any \(S \subseteq T \subseteq [A]\), we have
\[
h(S) \ge h(T).
\]
\end{lemma}

\begin{proof}
Recall that
\[
h(S) \;=\; \max_{\pi \in \Delta^A_{[A]\setminus S}} \prod_{j\in [M]} \left( \sum_{a \notin S} \pi_{j,a}\, \mu_{j,a} \right),
\]
where 
\[
\Delta^A_{[A]\setminus S} \;=\; \Bigl\{ \pi \in \mathbb{R}_+^{M \times A} \,\Bigm|\;
\pi_{j,a} = 0 \text{ for all } a\in S,\quad \sum_{a \in [A]\setminus S} \pi_{j,a} \le 1 \text{ for all } j \in [M]
\Bigr\}.
\]

If \(S \subseteq T\), then clearly
\[
[A]\setminus T \subseteq [A]\setminus S.
\]
Thus, any allocation \(\pi\) that is feasible for \(\Delta^A_{[A]\setminus T}\) (i.e., \(\pi_{j,a} = 0\) for all \(a\in T\) and \(\sum_{a \in [A]\setminus T} \pi_{j,a} \le 1\) for all \(j\)) is also feasible for \(\Delta^A_{[A]\setminus S}\).

To see this in detail, note that we can partition the set \([A]\setminus S\) as
\[
[A]\setminus S \;=\; \Bigl(([A]\setminus S) \cap T\Bigr) \cup \Bigl(([A]\setminus S) \cap ([A]\setminus T)\Bigr).
\]
Since \(\pi \in \Delta^A_{[A]\setminus T}\) implies that \(\pi_{j,a} = 0\) for all \(a\in T\), it follows that for any \(j\in [M]\),
\[
\sum_{a \in ([A]\setminus S) \cap T} \pi_{j,a} = 0.
\]
Therefore, we obtain
\[
\sum_{a \in [A]\setminus S} \pi_{j,a} \;=\; \sum_{a \in ([A]\setminus S) \cap T} \pi_{j,a} + \sum_{a \in ([A]\setminus S) \cap ([A]\setminus T)} \pi_{j,a} \;=\; \sum_{a \in ([A]\setminus S) \cap ([A]\setminus T)} \pi_{j,a}.
\]
Since \(([A]\setminus S) \cap ([A]\setminus T) \subseteq [A]\setminus T\), we have
\[
\sum_{a \in [A]\setminus S} \pi_{j,a} \le \sum_{a \in [A]\setminus T} \pi_{j,a} \le 1.
\]
This shows that any \(\pi \in \Delta^A_{[A]\setminus T}\) automatically satisfies the constraint for \(\Delta^A_{[A]\setminus S}\).

Since the maximization in the definition of \(h(S)\) is taken over the larger set \(\Delta^A_{[A]\setminus S}\) (which contains \(\Delta^A_{[A]\setminus T}\) as a subset), it follows that
\[
\begin{aligned}
h(S) &= \max_{\pi \in \Delta^A_{[A]\setminus S}} \prod_{j\in [M]} \left( \sum_{a \notin S} \pi_{j,a}\, \mu_{j,a} \right)\\[1mm]
&\ge \max_{\pi \in \Delta^A_{[A]\setminus T}} \prod_{j\in [M]} \left( \sum_{a \notin T} \pi_{j,a}\, \mu_{j,a} \right)
= h(T).
\end{aligned}
\]
This completes the proof.
\end{proof}

\begin{lemma}
For any probing set \(S\), we have
\[
\mathcal{R}(S) = (1 - \alpha(|S|)) \,\mathbb{E}_R\Bigl[\mathrm{NSW}(S, R, \bmu, \pi^*(S))\Bigr] \le (1 - \alpha(|S|))\Bigl( g(S) + h(S) \Bigr).
\]
\end{lemma}

\begin{proof}
Recall that by definition,
\[
\mathcal{R}(S) = (1 - \alpha(|S|))\, \mathbb{E}_R\Bigl[\mathrm{NSW}(S, R, \bmu, \pi^*(S))\Bigr],
\]
where the Nash Social Welfare (NSW) under the optimal allocation \(\pi^*(S)\) is given by
\[
\mathrm{NSW}(S, R, \bmu, \pi^*(S)) = \prod_{j\in [M]} \left( \sum_{a\in S} \pi^*_{j,a}(S) R_{j,a} + \sum_{a\notin S} \pi^*_{j,a}(S) \mu_{j,a} \right).
\]
For each agent \(j\), define
\[
X_j := \sum_{a\in S} \pi^*_{j,a}(S) R_{j,a} \quad \text{and} \quad Y_j := \sum_{a\notin S} \pi^*_{j,a}(S) \mu_{j,a}.
\]
Since \(R_{j,a},\,\mu_{j,a} \in [0,1]\) and
\[
\sum_{a\in S} \pi^*_{j,a}(S) + \sum_{a\notin S} \pi^*_{j,a}(S) \le 1,
\]
it follows that for each \(j\) we have \(X_j,\, Y_j \in [0,1]\) and \(X_j + Y_j \le 1\).

We now upper bound the expectation of the NSW by separating the contributions from the probed arms and the unprobed arms.

\medskip

\textbf{(1) Contribution from Probed Arms.}\\
Define the allocation \(\pi^S\) on \(S\) by
\[
\pi^S_{j,a} = \pi^*_{j,a}(S) \quad \text{for } a \in S.
\]
Then for each \(j\),
\[
\sum_{a\in S} \pi^S_{j,a} = \sum_{a\in S} \pi^*_{j,a}(S) \le 1,
\]
so that \(\pi^S\) is a feasible allocation in the set \(\Delta^A_S\). By the definition of \(g(S)\) (the maximum expected NSW attainable using only the probed arms), we have
\[
\mathbb{E}_R\left[\prod_{j\in [M]} \left(\sum_{a\in S} \pi^*_{j,a}(S) R_{j,a}\right)\right] \le g(S).
\]

\medskip

\textbf{(2) Contribution from Unprobed Arms.}\\
Similarly, define the allocation \(\pi^{\bar{S}}\) on the unprobed arms by
\[
\pi^{\bar{S}}_{j,a} = \pi^*_{j,a}(S) \quad \text{for } a \notin S.
\]
For each \(j\),
\[
\sum_{a\notin S} \pi^{\bar{S}}_{j,a} \le 1,
\]
so that \(\pi^{\bar{S}}\) is a feasible allocation in \(\Delta^A_{[A]\setminus S}\). By the definition of \(h(S)\) (the maximum NSW attainable using only the unprobed arms), we obtain
\[
\prod_{j\in [M]} \left(\sum_{a\notin S} \pi^*_{j,a}(S) \mu_{j,a}\right) \le h(S).
\]

\medskip

\textbf{(3) Upper Bound for the Mixed Allocation.}\\
Under the optimal allocation \(\pi^*(S)\), each agent \(j\) obtains a contribution of \(X_j + Y_j\) so that the overall NSW is
\[
\prod_{j\in [M]} (X_j + Y_j).
\]
We now show that
\[
\prod_{j\in [M]} (X_j + Y_j) \le \prod_{j\in [M]} X_j + \prod_{j\in [M]} Y_j.
\]

To see this, fix an agent $j$ and consider the function
\[
f_j(\lambda) = \ln\Bigl(\lambda X_j + (1-\lambda)Y_j\Bigr), \quad \lambda \in [0,1].
\]
Since the logarithm is strictly concave on $(0,1]$ and the mapping $\lambda \mapsto \lambda X_j + (1-\lambda)Y_j$ is linear, it follows that $f_j(\lambda)$ is strictly concave and attains its maximum at one of the endpoints, namely at $\lambda = 0$ or $\lambda = 1$. This means that for each $j$,
\[
\max\{X_j,\,Y_j\} \ge \lambda X_j + (1-\lambda)Y_j \quad \text{for all } \lambda\in[0,1].
\]
In other words, for each $j$ the best strategy within this one-dimensional mixture is to allocate all probability either to the contribution $X_j$ or to $Y_j$. In line with our additional assumption, when constructing an upper bound for the offline benchmark we therefore evaluate it through the two allocation classes that, for each agent $j$, choose either the probed contribution $X_j$ or the unprobed contribution $Y_j$. Thus, the maximum possible product in our analysis is achieved by one of the corresponding pure strategies:
\[
\prod_{j\in [M]} (X_j + Y_j) \le \max\Bigl\{ \prod_{j\in [M]} X_j,\; \prod_{j\in [M]} Y_j \Bigr\}.
\]

Since for any nonnegative numbers \(a\) and \(b\) we have \(\max\{a,b\} \le a+b\), it follows that
\[
\prod_{j\in [M]} (X_j + Y_j) \le \prod_{j\in [M]} X_j + \prod_{j\in [M]} Y_j.
\]
Taking expectations (noting that \(X_j\) depends on the randomness in \(R\) while \(Y_j\) is deterministic since it depends on \(\bmu\)), and using the bounds established in parts (1) and (2), we obtain
\[
\mathbb{E}_R\Bigl[\mathrm{NSW}(S, R, \bmu, \pi^*(S))\Bigr] \le g(S) + h(S).
\]

\medskip

\textbf{(4) Conclusion.}\\
Finally, since \(1 - \alpha(|S|) \ge 0\), multiplying the above inequality by \(1 - \alpha(|S|)\) yields
\[
\mathcal{R}(S) = (1-\alpha(|S|))\,\mathbb{E}_R\Bigl[\mathrm{NSW}(S, R, \bmu, \pi^*(S))\Bigr] \le (1-\alpha(|S|))\Bigl(g(S)+h(S)\Bigr).
\]
This completes the proof.
\end{proof}

\ignore{
\begin{theorem}\label{thm:offline-approx}
Let \( S^* \) be an optimal subset maximizing \( \mathcal{R}(S) \).
Then the set \( S^{\mathrm{pr}} \) returned by Algorithm~\ref{alg:offline-greedy}, for any $\zeta \geq 1$, satisfies
\[
\mathcal{R}(S^{\mathrm{pr}}) 
\;\;\ge\;
\frac{\,e-1\,}{\,2e-1\,}\; \frac{1}{\zeta} 
\mathcal{R}(S^*).
\]
\end{theorem}
}

\begin{theorem}\label{thm:offline-approx}
Let \( S^* \) be an optimal subset maximizing \( \mathcal{R}(S) \).
Let
\[
j^\star \in \arg\max_{i\in\{0,\ldots,I\}}
\left\{(1-\alpha(|S_i|)) f_{\mathrm{upper}}(S_i)\right\}.
\]
Define
\[
\zeta_0 \;:=\;
\frac{(1-\alpha(|S_{j^\star}|))\, f_{\mathrm{upper}}(S_{j^\star})}{\mathcal{R}(S_{j^\star})},
\]
and assume \(\zeta \ge \max\{1,\zeta_0\}\).
Then the set \( S^{\mathrm{pr}} \) returned by Algorithm~\ref{alg:offline-greedy} satisfies
\[
\mathcal{R}(S^{\mathrm{pr}})
\;\ge\;
\frac{e-1}{2e-1}\cdot\frac{1}{\zeta}\,\mathcal{R}(S^*).
\]
\end{theorem}

\begin{proof}
By Lemma~\ref{lemma:upper_bound}, for any set $S$ we have
\[
\mathcal{R}(S) \le (1-\alpha(|S|))\bigl[g(S)+h(S)\bigr].
\]
Applying this inequality to the optimal set $S^*$ yields
\[
\begin{aligned}
\mathcal{R}(S^*)
&\le (1-\alpha(|S^*|))\bigl[g(S^*)+h(S^*)\bigr] \\
&\le (1-\alpha(|S^*|))\,g(S^*) + h(\emptyset)
\quad \text{(by Lemma~\ref{lemma:monotonicity_h})} \\
&\le (1-\alpha(|S^*|))\,f_{\mathrm{upper}}(S^*) + h(\emptyset).
\end{aligned}
\]
Since $f_{\mathrm{upper}}$ is monotone submodular, the standard greedy guarantee
implies that for the greedy set $S_{|S^*|}$ of size $|S^*|$,
\[
f_{\mathrm{upper}}(S_{|S^*|}) \ge \left(1-\frac{1}{e}\right)\max_{|S|=|S^*|} f_{\mathrm{upper}}(S)
\;\;\ge\;\;
\left(1-\frac{1}{e}\right) f_{\mathrm{upper}}(S^*),
\]
which is equivalent to
\[
f_{\mathrm{upper}}(S^*) \le \frac{e}{e-1}\, f_{\mathrm{upper}}(S_{|S^*|}).
\]
Substituting this into the previous bound gives
\[
\mathcal{R}(S^*)
\le \frac{e}{e-1}(1-\alpha(|S^*|)) f_{\mathrm{upper}}(S_{|S^*|}) + h(\emptyset).
\]
Let $\widetilde S^*$ be a greedy candidate with $|\widetilde S^*|=|S^*|$ (so that
$\widetilde S^*=S_{|S^*|}$), hence the above inequality can be written as
\[
\mathcal{R}(S^*)
\le \frac{e}{e-1}(1-\alpha(|\widetilde S^*|)) f_{\mathrm{upper}}(\widetilde S^*) + h(\emptyset).
\]

Now define
\[
j^\star \in \arg\max_{i\in\{0,\ldots,I\}}
\left\{(1-\alpha(|S_i|)) f_{\mathrm{upper}}(S_i)\right\}.
\]
By the definition of $j^\star$, we have
\[
(1-\alpha(|\widetilde S^*|)) f_{\mathrm{upper}}(\widetilde S^*)
\le (1-\alpha(|S_{j^\star}|)) f_{\mathrm{upper}}(S_{j^\star}).
\]
Therefore,
\begin{equation}\label{eq:opt-upper-max}
\mathcal{R}(S^*)
\le \frac{e}{e-1}(1-\alpha(|S_{j^\star}|)) f_{\mathrm{upper}}(S_{j^\star}) + h(\emptyset).
\end{equation}

We now relate the bound \eqref{eq:opt-upper-max} to the set $S^{\mathrm{pr}}$
returned by Algorithm~\ref{alg:offline-greedy}.
Recall that the candidates $\{S_i\}_{i=0}^I$ are scanned in non-increasing order of
$(1-\alpha(|S_i|))f_{\mathrm{upper}}(S_i)$.
Let $j^\star$ be the maximizer defined in Theorem~\ref{thm:offline-approx}.
Under the assumption $\zeta \ge \zeta_0$, we have
\[
(1-\alpha(|S_{j^\star}|)) f_{\mathrm{upper}}(S_{j^\star})
\le \zeta\,\mathcal{R}(S_{j^\star}),
\]
so $S_{j^\star}$ passes the filtering condition in Line~11.
Therefore, the algorithm will return $S_{j^\star}$ once it is reached,
unless it terminates earlier at Line~10.

\paragraph{Case 1: $S^{\mathrm{pr}}=\emptyset$.}
If the algorithm returns $\emptyset$, it must terminate at Line~10 before accepting any set.
Since $S_{j^\star}$ is the first candidate in the sorted order, this implies
\[
(1-\alpha(|S_{j^\star}|)) f_{\mathrm{upper}}(S_{j^\star}) < h(\emptyset).
\]
Because $j^\star$ maximizes $(1-\alpha(|S|))f_{\mathrm{upper}}(S)$, we also have
\[
(1-\alpha(|S_i|)) f_{\mathrm{upper}}(S_i) \le (1-\alpha(|S_{j^\star}|)) f_{\mathrm{upper}}(S_{j^\star})
< h(\emptyset),\qquad \forall i.
\]
Substituting into \eqref{eq:opt-upper-max} yields
\[
\mathcal{R}(S^*)
<
\frac{e}{e-1}h(\emptyset)+h(\emptyset)
=
\left(\frac{e}{e-1}+1\right)h(\emptyset).
\]
Since $\mathcal{R}(\emptyset)=h(\emptyset)$ and $S^{\mathrm{pr}}=\emptyset$, we obtain
\[
\mathcal{R}(S^{\mathrm{pr}})
=\mathcal{R}(\emptyset)
\ge
\frac{1}{\frac{e}{e-1}+1}\,\mathcal{R}(S^*)
=
\frac{e-1}{2e-1}\,\mathcal{R}(S^*)
\ge
\frac{e-1}{2e-1}\cdot\frac{1}{\zeta}\,\mathcal{R}(S^*),
\]
where the last inequality uses $\zeta\ge 1$.

\paragraph{Case 2: $S^{\mathrm{pr}}\neq\emptyset$.}
In this case, by the argument above, the algorithm must return $S_{j^\star}$, i.e.,
\[
S^{\mathrm{pr}}=S_{j^\star}.
\]
Since $S^{\mathrm{pr}}$ passes Line~11, we have
\[
(1-\alpha(|S^{\mathrm{pr}}|)) f_{\mathrm{upper}}(S^{\mathrm{pr}})
\le \zeta\,\mathcal{R}(S^{\mathrm{pr}}).
\]
Moreover, since the algorithm did not terminate at Line~10 when it returned $S^{\mathrm{pr}}$,
we also have
\[
(1-\alpha(|S^{\mathrm{pr}}|)) f_{\mathrm{upper}}(S^{\mathrm{pr}}) \ge h(\emptyset).
\]
Combining the two inequalities gives
\[
h(\emptyset)\le \zeta\,\mathcal{R}(S^{\mathrm{pr}}).
\]
Substituting $S^{\mathrm{pr}}=S_{j^\star}$ and the filter inequality into \eqref{eq:opt-upper-max}
yields
\[
\mathcal{R}(S^*)
\le \frac{e}{e-1}\zeta\,\mathcal{R}(S^{\mathrm{pr}}) + h(\emptyset)
\le \left(\frac{e}{e-1}+1\right)\zeta\,\mathcal{R}(S^{\mathrm{pr}}).
\]
Rearranging completes the proof:
\[
\mathcal{R}(S^{\mathrm{pr}})
\ge \frac{1}{\left(\frac{e}{e-1}+1\right)\zeta}\,\mathcal{R}(S^*)
=
\frac{e-1}{2e-1}\cdot\frac{1}{\zeta}\,\mathcal{R}(S^*).
\]

\end{proof}

\subsection{Online Setting}

\begin{lemma}[Smoothness of the NSW Objective]\label{lem:nsw_smooth-appendix}
Let \(\pi \in \Delta^{M \times A}\) be any allocation matrix where \(\pi_{j,a}\) denotes the probability that agent \(j\) is assigned arm \(a\). Let \(\mu, \nu \in [0,1]^{M \times A}\) be two reward matrices. For any probing set \(S \subseteq [A]\) and observed rewards \(R\), define
\[
\mathrm{NSW}(S, R, \mu, \pi) = \prod_{j=1}^{M}\!\left( \sum_{a \in S} \pi_{j,a} R_{j,a} + \sum_{a \notin S} \pi_{j,a} \mu_{j,a} \right).
\]
Then, we have
\[
\begin{aligned}
\Bigl|\mathrm{NSW}(S, R, \mu, \pi) &- \mathrm{NSW}(S, R, \nu, \pi)\Bigr| \\
&\le \sum_{j=1}^{M} \sum_{a=1}^{A} \pi_{j,a} \Bigl|\mu_{j,a} - \nu_{j,a}\Bigr|.
\end{aligned}
\]
\end{lemma}

\begin{proof}
Define for each agent \(j\) the aggregated reward under the policy \(\pi\) as
\[
s_j := \sum_{a \in S} \pi_{j,a} R_{j,a} + \sum_{a \notin S} \pi_{j,a} \mu_{j,a}, \quad
t_j := \sum_{a \in S} \pi_{j,a} R_{j,a} + \sum_{a \notin S} \pi_{j,a} \nu_{j,a}.
\]
Then, the NSW objective can be written as
\[
\mathrm{NSW}(S, R, \mu, \pi) = \prod_{j=1}^{M} s_j, \quad
\mathrm{NSW}(S, R, \nu, \pi) = \prod_{j=1}^{M} t_j.
\]
Our goal is to show that
\[
\left|\prod_{j=1}^{M} s_j - \prod_{j=1}^{M} t_j\right| \le \sum_{j=1}^{M} \left|s_j - t_j\right|,
\]
and then use the triangle inequality to relate \(|s_j - t_j|\) to \(\sum_{a=1}^{A} \pi_{j,a} |\mu_{j,a} - \nu_{j,a}|\).

We proceed by induction on \(M\).

\textbf{Base case (\(M=1\)):}  
When \(M=1\), we have
\[
\left|s_1 - t_1\right| = \left|s_1 - t_1\right|,
\]
so the inequality holds with equality.

\textbf{Inductive step:}  
Assume the inequality holds for \(M-1\). For \(M\), write
\[
\left|\prod_{j=1}^{M} s_j - \prod_{j=1}^{M} t_j\right| = \left|s_M \prod_{j=1}^{M-1} s_j - t_M \prod_{j=1}^{M-1} t_j\right|.
\]
Adding and subtracting \(s_M \prod_{j=1}^{M-1} t_j\) yields
\[
\left|s_M\left(\prod_{j=1}^{M-1} s_j - \prod_{j=1}^{M-1} t_j\right) + \left(s_M - t_M\right)\prod_{j=1}^{M-1} t_j\right|.
\]
By the triangle inequality,
\[
\le |s_M| \cdot \left|\prod_{j=1}^{M-1} s_j - \prod_{j=1}^{M-1} t_j\right| + \left|s_M - t_M\right| \cdot \left|\prod_{j=1}^{M-1} t_j\right|.
\]
Since \(\mu_{j,a}\) and \(\nu_{j,a}\) lie in \([0,1]\) and \(\pi_{j,a}\) is a probability matrix, both \(s_j\) and \(t_j\) are convex combinations of numbers in \([0,1]\) and hence lie in \([0,1]\). In particular, \(|s_M| \le 1\) and \(\left|\prod_{j=1}^{M-1} t_j\right| \le 1\). Therefore,
\[
\left|\prod_{j=1}^{M} s_j - \prod_{j=1}^{M} t_j\right| \le \left|\prod_{j=1}^{M-1} s_j - \prod_{j=1}^{M-1} t_j\right| + \left|s_M - t_M\right|.
\]
By the inductive hypothesis,
\[
\left|\prod_{j=1}^{M-1} s_j - \prod_{j=1}^{M-1} t_j\right| \le \sum_{j=1}^{M-1} |s_j - t_j|.
\]
Hence,
\[
\left|\prod_{j=1}^{M} s_j - \prod_{j=1}^{M} t_j\right| \le \sum_{j=1}^{M-1} |s_j - t_j| + |s_M - t_M| = \sum_{j=1}^{M} |s_j - t_j|.
\]
Finally, by the triangle inequality applied to \(s_j - t_j\), we have
\[
|s_j - t_j| = \Bigl|\sum_{a \notin S} \pi_{j,a}\bigl(\mu_{j,a} - \nu_{j,a}\bigr)\Bigr| \le \sum_{a \notin S} \pi_{j,a}\, \Bigl|\mu_{j,a} - \nu_{j,a}\Bigr|.
\]
Summing over \(j\) and extending the sum from \(\sum_{a \notin S}\) to \(\sum_{a=1}^{A}\), we obtain:
\[
\sum_{j=1}^{M} |s_j - t_j| \le \sum_{j=1}^{M} \sum_{a=1}^{A} \pi_{j,a}\, \Bigl|\mu_{j,a} - \nu_{j,a}\Bigr|.
\]
This completes the proof.
\end{proof}

\begin{lemma}
\label{lem:4.2prime-appendix}
Let \(\delta \in (0,1)\). Then with probability at least \(1-\frac{\delta}{2}\), for all \(t > A\), \(a \in [A]\), and \(j \in [M]\), we have
\[
\bigl|\mu_{j,a} - \hat{\mu}_{j,a,t}\bigr|
\;\;\le\;
\sqrt{\frac{2(\hat{\mu}_{j,a,t} - \hat{\mu}_{j,a,t}^2) \ln\!\bigl(\tfrac{2 M A T}{\delta}\bigr)}{N_{j,a,t}}}
\;\;+
\frac{\ln\!\bigl(\tfrac{2 M A T}{\delta}\bigr)}{3 N_{j,a,t}}
\;=\;
w_{j,a,t},
\]
where \(N_{j,a,t} = \sum_{\tau=1}^{t-1} \mathbf{1}\{a_{j,\tau} = a\}\) denotes the number of times agent \(j\) has pulled arm \(a\) up to time \(t-1\).
\end{lemma}

\begin{proof}
\textbf{Step 1. Define key quantities and apply Freedman's inequality.} \\
Let \(\hat{\mu}_{j,a,t}\) be the empirical mean reward for agent \(j\) pulling arm \(a\) up to time \(t\):
\[
\hat{\mu}_{j,a,t} = \frac{1}{N_{j,a,t}}\sum_{\tau=1}^{t-1} R_{j,a,\tau} \mathbf{1}\{a_{j,\tau} = a\}.
\]
Define the sum of differences between observed rewards and the true mean:
\[
Z_{j,a,t} = \sum_{\tau=1}^{t-1} \Bigl(R_{j,a,\tau} \mathbf{1}\{a_{j,\tau} = a\} - \mu_{j,a} \mathbf{1}\{a_{j,\tau} = a\}\Bigr).
\]
Let \(\mathcal{F}_\tau\) be the information available up to time \(\tau\), and define the sum of conditional variances:
\[
V_{j,a,t} = \sum_{\tau=1}^{t-1} \text{Var}\Bigl(R_{j,a,\tau} \mathbf{1}\{a_{j,\tau} = a\} \mid \mathcal{F}_{\tau-1}\Bigr).
\]
Since rewards are bounded in \([0,1]\), we upper bound the variance as follows:
\begin{align*}
\text{Var}(R_{j,a,\tau} \mid \mathcal{F}_{\tau-1}) &= \mathbb{E}[R_{j,a,\tau}^2 \mid \mathcal{F}_{\tau-1}] - \hat{\mu}_{j,a,t}^2 \\
&\leq \mathbb{E}[R_{j,a,\tau} \mid \mathcal{F}_{\tau-1}] - \hat{\mu}_{j,a,t}^2 \\
&= \hat{\mu}_{j,a,t} - \hat{\mu}_{j,a,t}^2.
\end{align*}
Summing over all rounds where arm \(a\) was pulled:
\[
V_{j,a,t} \leq N_{j,a,t} (\hat{\mu}_{j,a,t} - \hat{\mu}_{j,a,t}^2).
\]
Applying Freedman's inequality:
\[
\Pr\Bigl\{
\bigl|Z_{j,a,t}\bigr|
\;\le\;
\sqrt{\,2\,V_{j,a,t}\,\ln\tfrac{2 M A T}{\delta}}
\;+
\tfrac{1}{3}\ln\tfrac{2 M A T}{\delta}
\Bigr\}
\;\;\ge\;\;1 - \frac{\delta}{2 M A T}.
\]
\textbf{Step 2. Apply Union Bound.} \\
Applying the union bound across all \( j \in [M] \), \( a \in [A] \), and \( t > A \):
\[
\Pr\Bigl\{
\forall j, a, t: \quad \bigl|Z_{j,a,t}\bigr|
\;\le\;
\sqrt{\,2\,V_{j,a,t}\,\ln\tfrac{2 M A T}{\delta}}
\;+
\tfrac{1}{3}\ln\tfrac{2 M A T}{\delta}
\Bigr\}
\;\;\ge\;\;1 - \frac{\delta}{2}.
\]
\textbf{Step 3. Relate to empirical mean.} \\
Dividing Freedman’s bound by \(N_{j,a,t}\):
\[
\bigl|\hat{\mu}_{j,a,t} - \mu_{j,a}\bigr|
\;\;\le\;
\sqrt{\frac{2 V_{j,a,t} \ln\!\bigl(\tfrac{2 M A T}{\delta}\bigr)}{N_{j,a,t}^2}}
\;+
\frac{\ln\!\bigl(\tfrac{2 M A T}{\delta}\bigr)}{3 N_{j,a,t}}.
\]
Since \(V_{j,a,t} \leq N_{j,a,t} (\hat{\mu}_{j,a,t} - \hat{\mu}_{j,a,t}^2)\), we obtain:
\[
\bigl|\mu_{j,a} - \hat{\mu}_{j,a,t}\bigr|
\;\;\le\;
\sqrt{\frac{2(\hat{\mu}_{j,a,t} - \hat{\mu}_{j,a,t}^2) \ln\!\bigl(\tfrac{2 M A T}{\delta}\bigr)}{N_{j,a,t}}}
\;+
\frac{\ln\!\bigl(\tfrac{2 M A T}{\delta}\bigr)}{3 N_{j,a,t}}.
\]
Applying a union bound across all \(t, a, j\) completes the proof.
\end{proof}

\begin{lemma}\label{lem:hat_mu_minus_square}
Let
\[
\gamma_{j,t} \;=\; \sum_{a \in [A]} \,\pi_{j,a,t}\,\bigl(1 - U_{j,a,t}\bigr),
\quad
Q(t,p) \;=\; \bigl\{\,j \in [M] : \gamma_{j,t} \,\ge\, 2^{-p}\bigr\}.
\]
Suppose that for all integers \(p \ge 0\), we have
\[
|Q(t,p)| < 2^p \cdot 3\,\ln T.
\]
Then the following inequality holds:
\[
\sum_{j=1}^{M} \sum_{a=1}^{A}
\,\pi_{j,a,t}\,\Bigl(\hat{\mu}_{j,a,t} \;-\; \hat{\mu}_{j,a,t}^2\Bigr)
\;\le\;
2 \; +\; 6\,\ln T\,\log M
\;+\;
\sum_{j=1}^{M} \sum_{a=1}^{A}
\,\pi_{j,a,t}\,\hat{\mu}_{j,a,t}\,w_{j,a,t}.
\]
\end{lemma}

\begin{proof}

\textbf{Step 1:}  
\[\hat{\mu}_{j,a,t} \;-\; \hat{\mu}_{j,a,t}^2
\;=\;
\hat{\mu}_{j,a,t}\,\bigl(1 - \hat{\mu}_{j,a,t}\bigr).
\]
From the confidence-bound definition,
\[
U_{j,a,t} \;=\; \hat{\mu}_{j,a,t} + w_{j,a,t},
\quad
1 - U_{j,a,t}
\;=\;
\bigl(1 - \hat{\mu}_{j,a,t}\bigr)
\;-\;
w_{j,a,t}.
\]
Hence,
\[
\hat{\mu}_{j,a,t}\,\bigl(1 - \hat{\mu}_{j,a,t}\bigr)
\;=\;
\hat{\mu}_{j,a,t}\,\Bigl[(1 - U_{j,a,t}) + w_{j,a,t}\Bigr]
\;=\;
\hat{\mu}_{j,a,t}\,\bigl(1 - U_{j,a,t}\bigr)
\;+\;
\hat{\mu}_{j,a,t}\,w_{j,a,t}.
\]
Therefore,
\[
\sum_{j=1}^{M}\sum_{a=1}^{A}
\pi_{j,a,t}\,\Bigl(\hat{\mu}_{j,a,t} - \hat{\mu}_{j,a,t}^2\Bigr)
\;=\;
\sum_{j=1}^{M}\sum_{a=1}^{A}
\pi_{j,a,t}\,\hat{\mu}_{j,a,t}\,\bigl(1 - U_{j,a,t}\bigr)
\;+\;
\sum_{j=1}^{M}\sum_{a=1}^{A}
\pi_{j,a,t}\,\hat{\mu}_{j,a,t}\,w_{j,a,t}.
\]

\bigskip

\textbf{Step 2: Relate it to \(\gamma_{j,t}\) and use \(\hat{\mu}_{j,a,t}\le 1\).}  
The lemma definition is
\[
\gamma_{j,t}
\;=\;
\sum_{a=1}^{A}
\pi_{j,a,t}\,\bigl(1 - U_{j,a,t}\bigr).
\]
Because \(\hat{\mu}_{j,a,t} \in [0,1]\), we have
\[
\sum_{a=1}^{A} \pi_{j,a,t}\,\hat{\mu}_{j,a,t}\,\bigl(1 - U_{j,a,t}\bigr)
\;\;\le\;\;
\sum_{a=1}^{A} \pi_{j,a,t}\,\bigl(1 - U_{j,a,t}\bigr)
\;=\;
\gamma_{j,t}.
\]
Summing over \(j\) from \(1\) to \(M\) yields
\[
\sum_{j=1}^{M}\sum_{a=1}^{A}
\pi_{j,a,t}\,\hat{\mu}_{j,a,t}\,\bigl(1 - U_{j,a,t}\bigr)
\;\;\le\;\;
\sum_{j=1}^{M}
\gamma_{j,t}.
\]
Thus,
\[
\begin{aligned}
\sum_{j=1}^{M}\sum_{a=1}^{A}
\pi_{j,a,t}\,\Bigl(\hat{\mu}_{j,a,t} - \hat{\mu}_{j,a,t}^2\Bigr)
&=\sum_{j=1}^{M}\sum_{a=1}^{A}
  \pi_{j,a,t}\,\hat{\mu}_{j,a,t}\,\bigl(1 - U_{j,a,t}\bigr)
  \;+\;
  \sum_{j=1}^{M}\sum_{a=1}^{A}
  \pi_{j,a,t}\,\hat{\mu}_{j,a,t}\,w_{j,a,t}\\[6pt]
&\le \sum_{j=1}^{M}\gamma_{j,t}
  \;+\;
  \sum_{j=1}^{M}\sum_{a=1}^{A}
  \pi_{j,a,t}\,\hat{\mu}_{j,a,t}\,w_{j,a,t}.
\end{aligned}
\]

\bigskip

\textbf{Step 3: Use the layer argument to bound \(\sum_{j=1}^{M}\gamma_{j,t}\).}  
By assumption,
\[
Q(t,p)
\;=\;
\bigl\{\,j : \gamma_{j,t} \,\ge\, 2^{-p}\bigr\},
\quad
|Q(t,p)| < 2^p \cdot 3\,\ln T.
\]
We partition it into
\[
Q'(t,p) \;=\; Q(t,p+1)\setminus Q(t,p),
\]
so for each \(j\in Q'(t,p)\),
\(
2^{-p-1} \le \gamma_{j,t} < 2^{-p}.
\)
Then
\[
\sum_{j\in Q'(t,p)} \gamma_{j,t}
\;\;\le\;\;
|Q(t,p+1)| \cdot 2^{-p}
\;\;\le\;\;
(2^{p+1}\cdot 3\,\ln T)\,2^{-p}
\;=\;
6\,\ln T.
\]
Summing over \(p = 0,1,\dots,\lfloor\log M\rfloor\), we split
\[
\sum_{j=1}^{M} \gamma_{j,t}
\;=\;
\sum_{j\notin Q(t,\lfloor\log M\rfloor)} \gamma_{j,t}
\;+\;
\sum_{p=0}^{\lfloor\log M\rfloor}
\sum_{j\in Q'(t,p)} \gamma_{j,t}.
\]
\textbf{First term:} If \(j \notin Q(t,\lfloor \log M\rfloor)\), then \(\gamma_{j,t} < 2^{-\lfloor \log M\rfloor} \le 2/M\). Thus:
\[
\sum_{j \notin Q(t,\lfloor \log M\rfloor)} \gamma_{j,t} \;\le\; 2.
\]
\textbf{Second term:} Each layer contributes at most \(6\ln T\), and there are at most \(\lfloor \log M\rfloor + 1\) layers:
\[
\sum_{p=0}^{\lfloor \log M\rfloor} \sum_{j \in Q'(t,p)} \gamma_{j,t} \;\le\; \sum_{p=0}^{\lfloor \log M\rfloor} 6\ln T \;\le\; 6\ln T\,\log M.
\]
Therefore:
\[
\sum_{j=1}^{M} \gamma_{j,t} \;\le\; 2 + 6\ln T\,\log M.
\]

\bigskip

\textbf{Step 4: Conclude the proof.}  
Putting everything together:
\[
\begin{aligned}
& \sum_{j=1}^{M}\sum_{a=1}^{A}
\pi_{j,a,t}\,\Bigl(\hat{\mu}_{j,a,t} - \hat{\mu}_{j,a,t}^2\Bigr)
& \le\;\;
& \sum_{j=1}^{M}\gamma_{j,t}
\;+\;
\sum_{j=1}^{M}\sum_{a=1}^{A}
\pi_{j,a,t}\,\hat{\mu}_{j,a,t}\,w_{j,a,t} \\
& & \le\;\;& 2 + 6\,\ln T\,\log M
\;+\;
\sum_{j=1}^{M}\sum_{a=1}^{A}
\pi_{j,a,t}\,\hat{\mu}_{j,a,t}\,w_{j,a,t}.
\end{aligned}
\]
This completes the proof.
\end{proof}

\begin{theorem}\label{thm:regret-appendix}
For any $\delta \in (0,1)$, with probability at least $1 - \delta$, the cumulative regret $\mathcal{R}_{\mathrm{regret}}(T)$ of the Online Fair multi-Agent UCB with Probing algorithm (Algorithm~2) satisfies
\[
  \mathcal{R}_{\mathrm{regret}}(T) \;=\; O\!\Bigl(\,\zeta\;\bigl(\sqrt{M\,A\,T} \;+\; M\,A\bigr)\,
               \ln^{c}\!\Bigl(\tfrac{M\,A\,T}{\delta}\Bigr)\Bigr),
\]
for some constant $c > 0$.
\end{theorem}

\begin{proof}
Recall that the effective (instantaneous) reward at round \(t\) is
\[
\mathcal{R}_t^{\mathrm{total}}
~:=~
\bigl(1-\alpha\bigl(|S_t|\bigr)\bigr)\,
\mathbb{E}\Bigl[
   \mathrm{NSW}\bigl(S_t, R_t, \bmu, \pi_t\bigr)
\Bigr],
\]
where the expectation \(\mathbb{E}[\cdot]\) is taken over the random reward realizations \(R_t\).  
Let \(\bigl(S_t^*, \pi_t^*\bigr)\) be the optimal probing set and allocation policy (with full knowledge of \(\bmu\)).  
The cumulative regret is then
\[
\mathcal{R}_{\mathrm{regret}}(T)
~:=~
\sum_{t=1}^T 
  \Bigl[
    \bigl(1-\alpha\bigl(|S_t^*|\bigr)\bigr)\,
      \mathbb{E}\bigl[\mathrm{NSW}(S_t^*, R_t, \bmu, \pi_t^*)\bigr]
    \;-\;
    \mathcal{R}_t^{\mathrm{total}}
  \Bigr].
\]
Since \(\bigl(1-\alpha(|S_t|)\bigr)\le1\) for all \(t\), bounding \(\mathcal{R}_{\mathrm{regret}}(T)\) up to a constant factor reduces to bounding
\[
\sum_{t=1}^T 
  \mathbb{E}\Bigl[
    \mathrm{NSW}(S_t^*, R_t, \bmu, \pi_t^*)
    \;-\;
    \mathrm{NSW}(S_t, R_t, \bmu, \pi_t)
  \Bigr].
\]

\paragraph{Step 1: Transition from \(\bigl(S_t^*, \pi_t^*\bigr)\) to \(\bigl(S_t, \pi_t\bigr)\).}
By Theorem~\ref{thm:offline-approx}, for each \(t\) there is a factor \(\rho=\frac{\,e-1\,}{\,(2e-1)\zeta}\) 
such that
\[
\mathrm{NSW}\bigl(S_t^*, R_t, \bmu, \pi\bigr)
\;\;\le\;\;
\tfrac{1}{\,\rho\,}\,
\mathrm{NSW}\bigl(S_t, R_t, \bmu, \pi\bigr),
\quad
\forall\,\pi.
\]
Hence,
\[
\sum_{t\in T}
  \mathbb{E}\Bigl[
    \mathrm{NSW}\bigl(S_t^*, R_t, \bmu, \pi_t^*\bigr)
    ~-~
    \mathrm{NSW}\bigl(S_t,   R_t, \bmu, \pi_t\bigr)
  \Bigr]
\;\;\le\;\;
\sum_{t\in T}
  \mathbb{E}\Bigl[
    \tfrac{1}{\,\rho\,}\,\mathrm{NSW}\bigl(S_t, R_t, \bmu, \pi_t^*\bigr)
    ~-~
    \mathrm{NSW}\bigl(S_t,   R_t, \bmu, \pi_t\bigr)
  \Bigr].
\tag{1}
\]
Next, since \(U_t\) is the upper-confidence estimate for \(\bmu\) at round \(t\),
\(U_{j,a,t}\ge \mu_{j,a}\) for each agent \(j\) and arm \(a\),
we have the monotonicity property:
\[
\mathrm{NSW}\bigl(S_t, R_t, \bmu, \pi_t^*\bigr)
~\le~
\mathrm{NSW}\bigl(S_t, R_t, U_t, \pi_t^*\bigr).
\]
Moreover, by Algorithm \ref{alg:online-fair-ucb}, \(\pi_t\) is chosen better than \(\pi_t^*\) under \(U_t\),
so 
\[
\mathrm{NSW}\bigl(S_t, R_t, U_t, \pi_t^*\bigr)
~\le~
\mathrm{NSW}\bigl(S_t, R_t, U_t, \pi_t\bigr).
\]
Combining these two inequalities yields
\[
\tfrac{1}{\,\rho\,}\,\mathrm{NSW}\bigl(S_t, R_t, \bmu, \pi_t^*\bigr)
~\le~
\tfrac{1}{\,\rho\,}\,\mathrm{NSW}\bigl(S_t, R_t, U_t, \pi_t^*\bigr)
~\le~
\tfrac{1}{\,\rho\,}\,\mathrm{NSW}\bigl(S_t, R_t, U_t, \pi_t\bigr).
\]
Substituting back into (1), we get
\[
\sum_{t\in T}
  \mathbb{E}\Bigl[
    \mathrm{NSW}\bigl(S_t^*, R_t, \bmu, \pi_t^*\bigr)
    ~-~
    \mathrm{NSW}\bigl(S_t,   R_t, \bmu, \pi_t\bigr)
  \Bigr]
\;\;\le\;\;
\sum_{t\in T}
  \mathbb{E}\Bigl[
    \tfrac{1}{\,\rho\,}\,\mathrm{NSW}\bigl(S_t, R_t, U_t, \pi_t\bigr)
    ~-~
    \mathrm{NSW}\bigl(S_t,   R_t, \bmu, \pi_t\bigr)
  \Bigr].
\]
\textbf{Therefore}, up to the factor \(\tfrac{1}{\rho}\),
in bounding the regret over \(K\), 
it suffices to handle
\[
\sum_{t\in T}
  \mathbb{E}\Bigl[
    \mathrm{NSW}\bigl(S_t, R_t, U_t, \pi_t\bigr)
    ~-~
    \mathrm{NSW}\bigl(S_t, R_t, \bmu, \pi_t\bigr)
  \Bigr].
\]

\paragraph{Step 2: Large-\(\gamma_{j,t}\) Rounds \((K')\).}

\noindent
\textbf{Warm-up phase and definition of \(\boldsymbol{K}\).}
Recall that the first \(M A\) rounds (indices \(t=1,\dots,M A\)) form a warm-up phase, 
where each agent-arm pair \((j,a)\) is explored exactly once. This phase incurs at most a constant regret \(M A\).  
From round \(t=M A+1\) onward, the algorithm operates under the UCB-based procedure of Step~1.  
Let
\[
K
~:=~
\{M A+1,M A+2,\dots,T\}
\]
denote these subsequent rounds.  
Thus, aside from the \(M A\)-cost warm-up, we only need to bound
\[
\sum_{t\in K} 
  \mathbb{E}\Bigl[
    \mathrm{NSW}\bigl(S_t, R_t, U_t, \pi_t\bigr)
    \;-\;
    \mathrm{NSW}\bigl(S_t, R_t, \bmu, \pi_t\bigr)
  \Bigr].
\]

\medskip
\noindent
\textbf{Definition of large-\(\gamma\) rounds \(\boldsymbol{K'}\).}
For each \(t\in K\) and integer \(p\ge0\), define
\[
Q(t,p)\;=\;\bigl\{\,j\in[M]:\,\gamma_{j,t}\,\ge\,2^{-p}\bigr\}
\quad
\text{where}
\quad
\gamma_{j,t}
~:=~
\sum_{a\in[A]} \pi_{j,a,t}\,\bigl(1 - U_{j,a,t}\bigr).
\]
Intuitively, \(\gamma_{j,t}\) measures the total “probability mass” on arms whose upper confidence \(U_{j,a,t}\) is significantly below~1.  
We collect the “large-\(\gamma\)” rounds into
\[
K' 
~:=~
\Bigl\{\,
  t\in K : \exists\,p\ge0 \text{ s.t.\ }\lvert Q(t,p)\rvert \,\ge\,2^p\cdot3\,\ln T
\Bigr\}.
\]
Thus \(K'\subseteq K\).

\medskip
\noindent
\textbf{Bounding the large-\(\gamma_{j,t}\) contribution.}
Observe that for each agent \(j\),
\[
\sum_{a\in S_t} \pi_{j,a,t}\,\cdot1
\;+\;
\sum_{a\notin S_t} \pi_{j,a,t}\,U_{j,a,t}
~=~
\sum_{a\in[A]} \pi_{j,a,t}\,\bigl[\mathbf{1}_{a\in S_t} + \mathbf{1}_{a\notin S_t}\,U_{j,a,t}\bigr].
\]
Using the identity
\(\mathbf{1}_{a\in S_t} + \mathbf{1}_{a\notin S_t}\,U_{j,a,t} = 1 - \mathbf{1}_{a\notin S_t}[\,1 - U_{j,a,t}\bigr]\),
it follows that
\[
\sum_{a\in[A]} \pi_{j,a,t}\,\bigl[\mathbf{1}_{a\in S_t} + \mathbf{1}_{a\notin S_t}\,U_{j,a,t}\bigr]
~\;\le\;
1
~-\;
\sum_{a=1}^A \pi_{j,a,t}\,\bigl[1 - U_{j,a,t}\bigr]
~=\;
1 - \gamma_{j,t}.
\]
Hence,
\[
\mathrm{NSW}\bigl(S_t,R_t,U_t,\pi_t\bigr)
~=\,
\prod_{j=1}^M
\Bigl[\sum_{a\in S_t}\!\pi_{j,a,t}\,R_{j,a,t} 
      +\sum_{a\notin S_t}\!\pi_{j,a,t}\,U_{j,a,t}\Bigr]
~\;\le\;
\prod_{j=1}^M\!\bigl(1-\gamma_{j,t}\bigr).
\]
Whenever \(\lvert Q(t,p)\rvert\) is large (i.e.\ many agents have \(\gamma_{j,t}\ge 2^{-p}\)), the product \(\prod_{j\in Q(t,p)}(1-\gamma_{j,t})\) becomes extremely small.

Therefore,
\[
\sum_{t\in K}
  \mathbb{E}\Bigl[
    \mathrm{NSW}\bigl(S_t,R_t,U_t,\pi_t\bigr)
    ~-~
    \mathrm{NSW}\bigl(S_t,R_t,\bmu,\pi_t\bigr)
  \Bigr]
~\;\le\;
\sum_{t\in K'} 
  \mathbb{E}\Bigl[
    \mathrm{NSW}\bigl(S_t,R_t,U_t,\pi_t\bigr)
  \Bigr]
~+\;
\sum_{t\in K\setminus K'}\ldots
\]
and focusing on \(K'\),
\[
\sum_{t\in K'}\!
  \mathbb{E}\Bigl[\mathrm{NSW}\bigl(S_t,R_t,U_t,\pi_t\bigr)\Bigr]
~\;\le\;
\sum_{t\in K'}\!
  \mathbb{E}\Bigl[\prod_{j\in Q(t,p)}\bigl(1-\gamma_{j,t}\bigr)\Bigr]
~\;\le\;
T\,\bigl(1-2^{-p}\bigr)^{2^p\,3\,\ln T}
~\;\le\;
\tfrac{1}{\,T^2\,},
\]
where we used \(\bigl(1-\tfrac{1}{x}\bigr)^x\le e^{-1}\) for \(x\ge1\).

\medskip
\noindent
\textbf{Conclusion of Step~2.}
Thus, summing over large-\(\gamma\) rounds costs at most \(\tfrac{1}{T^2}\).  
Including the warm-up cost \(M A\) from the first \(M A\) rounds, we obtain
\begin{equation}
\begin{split}
\Bigl[\text{Cost from } t \in [M A]\Bigr] 
&\;+\; \sum_{t\in K} \mathbb{E}\Bigl[
    \mathrm{NSW}\bigl(S_t,R_t,U_t,\pi_t\bigr)
    - \mathrm{NSW}\bigl(S_t,R_t,\bmu,\pi_t\bigr)
\Bigr] \\
&\;\le\; M A + \frac{1}{T^2} \\
&\quad + \sum_{t\in K\setminus K'} \mathbb{E}\Bigl[
    \mathrm{NSW}\bigl(S_t,R_t,U_t,\pi_t\bigr)
    - \mathrm{NSW}\bigl(S_t,R_t,\bmu,\pi_t\bigr)
\Bigr].
\end{split}
\end{equation}
Hence, it remains to bound the small-\(\gamma\) rounds \(t\in K\setminus K'\), 
which will be handled in Step~3.

\paragraph{Step 3: Small-\(\gamma_{j,t}\) Rounds \(\boldsymbol{(K\setminus K')}\).}

Recall from Step~2 that, 
it remains to bound
\[
\sum_{t\in K\setminus K'} 
  \mathbb{E}\Bigl[
    \mathrm{NSW}\bigl(S_t,R_t,U_t,\pi_t\bigr)
    \;-\;
    \mathrm{NSW}\bigl(S_t,R_t,\bmu,\pi_t\bigr)
  \Bigr].
\]
On the high-probability event ensured by Lemma~\ref{lem:4.2prime}, we have for every \(j,a,t\):
\begin{equation}
\bigl\lvert \mu_{j,a} - \hat{\mu}_{j,a,t}\bigr\rvert
\;\le\;
\underbrace{\sqrt{\frac{2\,\bigl(\hat{\mu}_{j,a,t} - \hat{\mu}_{j,a,t}^2\bigr)
   \,\ln\!\bigl(\tfrac{2 M A T}{\delta}\bigr)}{N_{j,a,t}}}}_{\text{root term}}
\;\;+\;\;
\underbrace{\frac{\ln\!\bigl(\tfrac{2 M A T}{\delta}\bigr)}{3\,N_{j,a,t}}}_{\text{linear term}}
\;\;=\; w_{j,a,t}.
\label{eq:w-definition}
\end{equation}
Moreover, by definition \(U_{j,a,t} 
= \min\bigl\{\hat{\mu}_{j,a,t} + w_{j,a,t},\,1\bigr\}\),
so 
\(\lvert U_{j,a,t} - \hat{\mu}_{j,a,t}\rvert \le w_{j,a,t}\).
Hence, for each \(t\in K\setminus K'\),
\[
\begin{aligned}
& \mathrm{NSW}\bigl(S_t,R_t,U_t,\pi_t\bigr) - \mathrm{NSW}\bigl(S_t,R_t,\bmu,\pi_t\bigr) 
& \le\; & \sum_{j\in[M]}\sum_{a\in[A]} \pi_{j,a,t} \Bigl\lvert U_{j,a,t} - \mu_{j,a}\Bigr\rvert \\
& & \le\; & \sum_{j,a}\,\pi_{j,a,t}\,\Bigl(\bigl\lvert U_{j,a,t}-\hat{\mu}_{j,a,t}\bigr\rvert + \bigl\lvert \hat{\mu}_{j,a,t}-\mu_{j,a}\bigr\rvert\Bigr),
\end{aligned}
\]
which is at most
\[
2\sum_{j,a}\,\pi_{j,a,t}\,w_{j,a,t}
~\;\le\;
2\sum_{j,a}\,\pi_{j,a,t}
\Bigl[\!\sqrt{\frac{2\,(\hat{\mu}_{j,a,t}-\hat{\mu}_{j,a,t}^2)\,\ln(\tfrac{2\,M\,A\,T}{\delta})}{\,N_{j,a,t}\,}}
\;+\;
\frac{\ln(\tfrac{2\,M\,A\,T}{\delta})}{3\,N_{j,a,t}}
\Bigr].
\]
Summing this over \(t\in K\setminus K'\) gives
\begin{align}
&\sum_{t\in K\setminus K'} 
  \mathbb{E}\Bigl[
    \mathrm{NSW}(S_t,R_t,U_t,\pi_t)
    ~-~
    \mathrm{NSW}(S_t,R_t,\bmu,\pi_t)
  \Bigr]
\nonumber\\
&\quad\le
2\sum_{t\in K\setminus K'}\sum_{j,a}
  \pi_{j,a,t}\,\sqrt{\frac{2\,(\hat{\mu}_{j,a,t}-\hat{\mu}_{j,a,t}^2)\,\ln(\tfrac{2\,M\,A\,T}{\delta})}{N_{j,a,t}}}
~+~
2\sum_{t\in K\setminus K'}\sum_{j,a}
  \pi_{j,a,t}\,\frac{\ln(\tfrac{2\,M\,A\,T}{\delta})}{3\,N_{j,a,t}}
\nonumber\\
&\quad=:~\textsf{(RootTerm)} + \textsf{(LinearTerm)}.
\label{eq:step3-splitting}
\end{align}

\medskip
\noindent
\textbf{Bounding the linear part.}
By \citep[Lemma~4.4]{jones2023efficient}, 
there is an event of probability at least \(1-\tfrac{\delta}{2}\) on which
\begin{equation}
\sum_{t=1}^{T} \,\sum_{a\in[A]}
  \frac{\pi_{j,a,t}}{N_{j,a,t}}
~\le~
2\,A\,\bigl(\ln\tfrac{T}{A} + 1\bigr)
\;+\;
\ln\!\bigl(\tfrac{2}{\delta}\bigr).
\label{eq:lemma44}
\end{equation}
Hence,
\[
\textsf{(LinearTerm)}
~=~
2\,\sum_{t\in K\setminus K'}
   \sum_{j,a} 
     \pi_{j,a,t}\,
     \frac{\ln(\tfrac{2\,M\,A\,T}{\delta})}{3\,N_{j,a,t}}
~\le~
\frac{2\,\ln(\tfrac{2\,M\,A\,T}{\delta})}{3}\,
\sum_{t=1}^T \sum_{j,a}
  \frac{\pi_{j,a,t}}{N_{j,a,t}}
\]
\[
\quad\le
\frac{2\,\ln(\tfrac{2\,M\,A\,T}{\delta})}{3}\,M \Bigl[
  2\,A\,(\ln\tfrac{T}{A}+1) \;+\; \ln\!\bigl(\tfrac{2}{\delta}\bigr)
\Bigr],
\]

\textbf{Bounding the root term.} We begin with
\[
\textsf{(RootTerm)} \;=\; 2\sum_{t\in K\setminus K'}\sum_{j,a} \pi_{j,a,t}\,\sqrt{\frac{2\Bigl(\hat{\mu}_{j,a,t}-\hat{\mu}_{j,a,t}^2\Bigr)L}{N_{j,a,t}}},
\]
with
\[
L \;=\; \ln\!\Bigl(\frac{2MAT}{\delta}\Bigr).
\]
Since
\[
\sqrt{\frac{\hat{\mu}_{j,a,t}-\hat{\mu}_{j,a,t}^2}{N_{j,a,t}}}
\;=\; \sqrt{\hat{\mu}_{j,a,t}-\hat{\mu}_{j,a,t}^2}\,\sqrt{\frac{1}{N_{j,a,t}}},
\]
we apply Young's inequality,
\[
a\,b \;\le\; \frac{q\,a^2}{2} + \frac{b^2}{2q}\quad (\text{for any }q>0),
\]
with
\[
a = \sqrt{\hat{\mu}_{j,a,t}-\hat{\mu}_{j,a,t}^2},\quad b = \sqrt{\frac{1}{N_{j,a,t}}},
\]
to obtain
\[
\sqrt{\hat{\mu}_{j,a,t}-\hat{\mu}_{j,a,t}^2}\,\sqrt{\frac{1}{N_{j,a,t}}}
\;\le\; \frac{q}{2}\Bigl(\hat{\mu}_{j,a,t}-\hat{\mu}_{j,a,t}^2\Bigr) + \frac{1}{2q}\,\frac{1}{N_{j,a,t}}.
\]
Thus,
\[
\textsf{(RootTerm)} \;\le\; \sqrt{2L}\sum_{t\in K\setminus K'}\sum_{j,a} \pi_{j,a,t}\left[q\,\Bigl(\hat{\mu}_{j,a,t}-\hat{\mu}_{j,a,t}^2\Bigr) + \frac{1}{q}\,\frac{1}{N_{j,a,t}}\right].
\]
\medskip

\noindent\textbf{(1) Application of Lemma~\ref{lem:hat_mu_minus_square}:}\\[1mm]
For each \(t\),
\[
\sum_{j,a}\pi_{j,a,t}\Bigl(\hat{\mu}_{j,a,t}-\hat{\mu}_{j,a,t}^2\Bigr)
\;\le\; 2+6\ln T\,\log M + \sum_{j,a}\pi_{j,a,t}\,\hat{\mu}_{j,a,t}\,w_{j,a,t}.
\]
Summing over \(t\in K\setminus K'\) yields
\[
\sum_{t\in K\setminus K'}\sum_{j,a}\pi_{j,a,t}\Bigl(\hat{\mu}_{j,a,t}-\hat{\mu}_{j,a,t}^2\Bigr)
\le \sum_{t\in K\setminus K'}\Bigl[2+6\ln T\,\log M\Bigr] + \sum_{t\in K\setminus K'}\sum_{j,a}\pi_{j,a,t}\,\hat{\mu}_{j,a,t}\,w_{j,a,t}.
\]
\medskip

\noindent\textbf{(2) Bounding the term containing $N_{j,a,t}$}\\[1mm]
% \red{To bound the summation term containing $N_{j,a,t}$, we leverage a result from prior work. With probability at least $1-\frac{\delta}{2}$, we have:
% \citep[Lemma~4.4]{jones2023efficient}}
To continue the analysis, we need to bound the summation term containing $N_{j,a,t}$. According to a result from previous work, on an event with probability at least $1-\frac{\delta}{2}$, the following inequality holds \citep[Lemma~4.4]{jones2023efficient}:
\[
    \sum_{t=1}^{T}\sum_{a\in[A]} \frac{\pi_{j,a,t}}{N_{j,a,t}}
    \le 2A\Bigl(\ln\frac{T}{A}+1\Bigr) + \ln\frac{2}{\delta}.
\]
Thus,
\[
    \sum_{t\in K\setminus K'}\sum_{j,a}\frac{\pi_{j,a,t}}{N_{j,a,t}}
    \le M\Bigl[2A\Bigl(\ln\frac{T}{A}+1\Bigr) + \ln\frac{2}{\delta}\Bigr].
\]

\noindent\textbf{(3) Further bound on the \(w\)-term:}\\[1mm]
Since
\[
w_{j,a,t} = \sqrt{\frac{2\Bigl(\hat{\mu}_{j,a,t}-\hat{\mu}_{j,a,t}^2\Bigr)L}{N_{j,a,t}}} + \frac{L}{3\,N_{j,a,t}},
\]
and using \(\hat{\mu}_{j,a,t}-\hat{\mu}_{j,a,t}^2\le 1\), we have
\[
\hat{\mu}_{j,a,t}\,w_{j,a,t} \le \sqrt{\frac{2L}{N_{j,a,t}}} + \frac{L}{3\,N_{j,a,t}}.
\]
Thus, by summing over \(t,j,a\) and applying \citep[Lemma~4.4]{jones2023efficient} (and Cauchy--Schwarz for the square-root term), one obtains an upper bound of the form
\[
\sum_{t\in K\setminus K'}\sum_{j,a}\pi_{j,a,t}\,\hat{\mu}_{j,a,t}\,w_{j,a,t}
\le M\Biggl[\sqrt{2L}\,\sqrt{2A\Bigl(\ln\frac{T}{A}+1\Bigr)+\ln\frac{2}{\delta}} + \frac{L}{3}\Bigl(2A\Bigl(\ln\frac{T}{A}+1\Bigr)+\ln\frac{2}{\delta}\Bigr)\Biggr].
\]
\medskip

\noindent\textbf{(4) Final Bound:}\\[1mm]
Substituting the bounds from (1), (2) and (3) into our original inequality gives
\[
\begin{aligned}
\textsf{(RootTerm)} \;\le\; \sqrt{2L}\Biggl\{\,q\Biggl[T\Bigl(2+6\ln T\,\log M\Bigr) + M\Biggl(\sqrt{2L}\,\sqrt{2A\Bigl(\ln\frac{T}{A}+1\Bigr)+\ln\frac{2}{\delta}}\\[1mm]
\quad\quad\quad\quad\quad\quad\quad\quad + \frac{L}{3}\Bigl(2A\Bigl(\ln\frac{T}{A}+1\Bigr)+\ln\frac{2}{\delta}\Bigr)\Biggr)\Biggr] + \frac{M}{q}\Bigl[2A\Bigl(\ln\frac{T}{A}+1\Bigr)+\ln\frac{2}{\delta}\Bigr]\Biggr\}.
\end{aligned}
\]
Choosing
\[
q = \frac{\sqrt{A\,M}}{(\sqrt{A\,M}+\sqrt{T})\,\sqrt{L}},
\]
one obtains (after some algebra) the final bound
\[
\textsf{(RootTerm)} = O\!\Biggl(\Bigl(L\Bigr)^{3/2}\Bigl(1+\frac{\sqrt{T}}{\sqrt{A\,M}}\Bigr)
\Bigl[MA\Bigl(\ln\frac{T}{A}+1\Bigr)+ M\ln\frac{2}{\delta}\Bigr] + \frac{T\sqrt{A\,M}}{\sqrt{A\,M}+\sqrt{T}}\Biggr).
\]
Recalling that \(L=\ln\!\Bigl(\frac{2MAT}{\delta}\Bigr)\), we finally have
\[
\textsf{(RootTerm)} = O\!\Biggl(\Bigl(\ln\frac{2MAT}{\delta}\Bigr)^{3/2}\Bigl(1+\frac{\sqrt{T}}{\sqrt{A\,M}}\Bigr)
\Bigl[MA\Bigl(\ln\frac{T}{A}+1\Bigr)+ M\ln\frac{2}{\delta}\Bigr] + \frac{T\sqrt{A\,M}}{\sqrt{A\,M}+\sqrt{T}}\Biggr).
\]
\medskip

This completes the derivation of the precise bound, which (after further simplification) immediately implies the desired Big‑O bound.
 
\bigskip

\textbf{Conclusion of Step~2.}

Thus, summing over large-\(\gamma\) rounds costs at most \(\frac{1}{T^2}\). Including the warm-up cost \(A\) from the first \(A\) rounds, we obtain
\begin{equation*}
\begin{split}
\Bigl[\text{Cost from } t \in [A]\Bigr] 
&\;+\; \sum_{t\in K} \mathbb{E}\Bigl[
    \mathrm{NSW}\bigl(S_t,R_t,U_t,\pi_t\bigr)
    - \mathrm{NSW}\bigl(S_t,R_t,\bmu,\pi_t\bigr)
\Bigr] \\
&\;\le\; MA + \frac{1}{T^2} \\
&\quad + \sum_{t\in K\setminus K'} \mathbb{E}\Bigl[
    \mathrm{NSW}\bigl(S_t,R_t,U_t,\pi_t\bigr)
    - \mathrm{NSW}\bigl(S_t,R_t,\bmu,\pi_t\bigr)
\Bigr].
\end{split}
\end{equation*}
Then, summing this over \(t\in K\setminus K'\) gives
\begin{align*}
&\sum_{t\in K\setminus K'} 
  \mathbb{E}\Bigl[
    \mathrm{NSW}(S_t,R_t,U_t,\pi_t)
    - \mathrm{NSW}(S_t,R_t,\bmu,\pi_t)
  \Bigr]
\nonumber\\[1mm]
&\quad\le
2\sum_{t\in K\setminus K'}\sum_{j,a}
  \pi_{j,a,t}\,\sqrt{\frac{2\,(\hat{\mu}_{j,a,t}-\hat{\mu}_{j,a,t}^2)\,L}{N_{j,a,t}}}
\;+\;
2\sum_{t\in K\setminus K'}\sum_{j,a}
  \pi_{j,a,t}\,\frac{L}{3\,N_{j,a,t}}
\nonumber\\[1mm]
&\quad=:~\textsf{(RootTerm)} + \textsf{(LinearTerm)},
\label{eq:step3-splitting}
\end{align*}
where
\[
L \;=\; \ln\!\Bigl(\frac{2MAT}{\delta}\Bigr).
\]
Using the fact that \(\hat{\mu}_{j,a,t}-\hat{\mu}_{j,a,t}^2\le 1\) and applying
\citep[Lemma~4.4]{jones2023efficient} yields
\[
\textsf{(LinearTerm)}
= 2\sum_{t\in K\setminus K'}\sum_{j,a}\frac{\pi_{j,a,t}\,L}{3\,N_{j,a,t}}
\;\le\; \frac{2L}{3}\,M\Bigl[2A\Bigl(\ln\frac{T}{A}+1\Bigr) + \ln\frac{2}{\delta}\Bigr].
\]
Meanwhile, by further decomposing \(w_{j,a,t}\) via
\[
w_{j,a,t} = \sqrt{\frac{2(\hat{\mu}_{j,a,t}-\hat{\mu}_{j,a,t}^2)L}{N_{j,a,t}}} + \frac{L}{3\,N_{j,a,t}},
\]
and using \(\hat{\mu}_{j,a,t}-\hat{\mu}_{j,a,t}^2\le 1\), one may show that
\[
\textsf{(RootTerm)}
= O\!\Biggl(\,L^{3/2}\Bigl(1+\frac{\sqrt{T}}{\sqrt{A\,M}}\Bigr)
\Bigl[MA\Bigl(\ln\frac{T}{A}+1\Bigr)+ M\ln\frac{2}{\delta}\Bigr]
+\; \frac{T\sqrt{A\,M}}{\sqrt{A\,M}+\sqrt{T}}\Biggr).
\]
Thus, the overall regret (temporarily ignoring \(\zeta\) and constant factors) is bounded by
\[
\begin{aligned}
\mathcal{R}_{\mathrm{regret}}(T)\le\; & MA \;+\; \frac{1}{T^2} \\
&+\; O\!\Biggl(\,L^{3/2}\Bigl(1+\frac{\sqrt{T}}{\sqrt{A\,M}}\Bigr)
\Bigl[MA\Bigl(\ln\frac{T}{A}+1\Bigr)+ M\ln\frac{2}{\delta}\Bigr]
+\; \frac{T\sqrt{A\,M}}{\sqrt{A\,M}+\sqrt{T}}\Biggr)\\[1mm]
&+\; \frac{2L}{3}\,M\Bigl[2A\Bigl(\ln\frac{T}{A}+1\Bigr)+ \ln\frac{2}{\delta}\Bigr].
\end{aligned}
\]
Recalling that \(L=\ln\!\Bigl(\frac{2MAT}{\delta}\Bigr)\), in Big‑O notation we have
\[
\mathcal{R}_{\mathrm{regret}}(T) = O\!\Biggl(MA \;+\; \Bigl(\ln\frac{2MAT}{\delta}\Bigr)^{3/2}\Bigl(1+\frac{\sqrt{T}}{\sqrt{A\,M}}\Bigr)
\Bigl[MA\Bigl(\ln\frac{T}{A}+1\Bigr)+ M\ln\frac{2}{\delta}\Bigr] \;+\; \frac{T\sqrt{A\,M}}{\sqrt{A\,M}+\sqrt{T}}\Biggr).
\]
After further simplification (and absorbing lower-order logarithmic terms), we deduce that
\[
\mathcal{R}_{\mathrm{regret}}(T) = O\Bigl((\sqrt{MAT}+MA)\,\ln^{c}\!\Bigl(\frac{MAT}{\delta}\Bigr)\Bigr),
\]
for some constant \(c>0\). 

The simplifying process is:
Recalling that 
\[
L \;=\; \ln\!\Bigl(\frac{2MAT}{\delta}\Bigr),
\]
our earlier derivation gave
\[
\begin{aligned}
\mathcal{R}_{\mathrm{regret}}(T)\le\; & MA \;+\; \frac{1}{T^2} \\[1mm]
&+\; O\!\Biggl(\,L^{3/2}\Bigl(1+\frac{\sqrt{T}}{\sqrt{A\,M}}\Bigr)
\Bigl[MA\Bigl(\ln\frac{T}{A}+1\Bigr)+ M\ln\frac{2}{\delta}\Bigr]
+\; \frac{T\sqrt{A\,M}}{\sqrt{A\,M}+\sqrt{T}}\Biggr)\\[1mm]
&+\; \frac{2L}{3}\,M\Bigl[2A\Bigl(\ln\frac{T}{A}+1\Bigr)+ \ln\frac{2}{\delta}\Bigr].
\end{aligned}
\]
Observing that the bracketed term is at most \(O(MA\,L)\) and that
\[
L^{3/2}\Bigl(1+\frac{\sqrt{T}}{\sqrt{A\,M}}\Bigr) \cdot O(MA\,L)
= O\Bigl(MA\,L^{5/2}\Bigl(1+\frac{\sqrt{T}}{\sqrt{A\,M}}\Bigr)\Bigr),
\]
and noting that 
\[
\frac{T\sqrt{A\,M}}{\sqrt{A\,M}+\sqrt{T}} = O\Bigl(\sqrt{T\,A\,M}\Bigr),
\]

we deduce that
\[
\mathcal{R}_{\mathrm{regret}}(T) = O\Bigl(MA\,L^{5/2} + \sqrt{TAM}\,L^{5/2}\Bigr)
= O\Bigl((MA + \sqrt{TAM})\,L^{5/2}\Bigr).
\]
Since \(L=\ln\!\Bigl(\frac{2MAT}{\delta}\Bigr)\), this immediately implies (after absorbing constant factors, multiplying by \(\zeta\), and replacing \(2MAT\) by \(MAT\) inside the logarithm)
\[
\mathcal{R}_{\mathrm{regret}}(T) = O\Bigl(\zeta(\sqrt{MAT}+MA)\,\ln^{c}\!\Bigl(\frac{MAT}{\delta}\Bigr)\Bigr),
\]
for some constant \(c>0\).

\end{proof}

\end{document}